\documentclass{article}

\usepackage{microtype}
\usepackage{graphicx}
\usepackage{subfigure}
\usepackage{booktabs} 

\usepackage[utf8]{inputenc} 
\usepackage[T1]{fontenc}    
\usepackage{hyperref}       
\usepackage{url}            
\usepackage{booktabs}       
\usepackage{amsfonts}       
\usepackage{amssymb} 
\usepackage{nicefrac,natbib}   
\usepackage{microtype}      
\usepackage{todonotes}
\usepackage{microtype,enumitem}
\usepackage{mathtools} 
\usepackage{amsmath}
\usepackage{bm}
\usepackage{algorithm2e}
\usepackage{algorithm}
\usepackage{amsthm}
\pdfoutput=1

\newtheorem{defn}{Definition}
\newtheorem{prop}{Proposition}

\newtheorem{rmq}{Remark}

\newtheorem{assumption}{Assumption}

\newtheorem*{example*}{Example}
\newtheorem*{prop*}{Proposition}
\newtheorem*{thm*}{Theorem}
\newtheorem*{prv*}{Proof}
\newtheorem*{defn*}{Definition}

\DeclareMathOperator{\Diag}{diag}

 \usepackage{wrapfig, framed, caption}
\usepackage{multirow}

\DeclareMathOperator*{\argmin}{argmin}

\def\rank{\textup{rk}}
\newcommand{\eqdef}{:=}
\renewcommand{\eqdef}{\triangleq}

\usepackage{hyperref}

\usepackage[accepted]{icml2021}

\icmltitlerunning{Submission and Formatting Instructions for ICML 2021}

\begin{document}

\twocolumn[
\icmltitle{Low-Rank Sinkhorn Factorization}



\icmlsetsymbol{equal}{*}

\begin{icmlauthorlist}
\icmlauthor{Meyer Scetbon}{ensae}
\icmlauthor{Marco Cuturi}{ensae,goo}
\icmlauthor{Gabriel Peyré}{ens,cnrs}
\end{icmlauthorlist}

\icmlaffiliation{ensae}{CREST, ENSAE}
\icmlaffiliation{goo}{Google Brain}
\icmlaffiliation{ens}{Ecole Normale Supérieure, PSL University}
\icmlaffiliation{cnrs}{CNRS}

\icmlcorrespondingauthor{Meyer Scetbon}{meyer.scetbon@ensae.fr}
\icmlcorrespondingauthor{Marco cuturi}{cuturi@google.com}
\icmlcorrespondingauthor{Gabriel Peyré}{gabriel.peyre@ens.fr}

\icmlkeywords{Machine Learning, ICML}

\vskip 0.3in
]



\printAffiliationsAndNotice{}  

\begin{abstract}
Several recent applications of optimal transport (OT) theory to machine learning have relied on regularization, notably entropy and the Sinkhorn algorithm. Because matrix-vector products are pervasive in the Sinkhorn algorithm, several works have proposed to \textit{approximate} kernel matrices appearing in its iterations using low-rank factors. Another route lies instead in imposing low-rank constraints on the feasible set of couplings considered in OT problems, with no approximations on cost nor kernel matrices. This route was first explored by~\citet{forrow2018statistical}, who proposed an algorithm tailored for the squared Euclidean ground cost, using a proxy objective that can be solved through the machinery of regularized 2-Wasserstein barycenters. Building on this, we introduce in this work a generic approach that aims at solving, in full generality, the OT problem under low-rank constraints with arbitrary costs. Our algorithm relies on an explicit factorization of low rank couplings as a product of \textit{sub-coupling} factors linked by a common marginal; similar to an NMF approach, we alternatively updates these factors. We prove the non-asymptotic stationary convergence of this algorithm and illustrate its efficiency on benchmark experiments.
\end{abstract}

\begin{figure*}[!ht]
\centering
\includegraphics[width=1\textwidth]{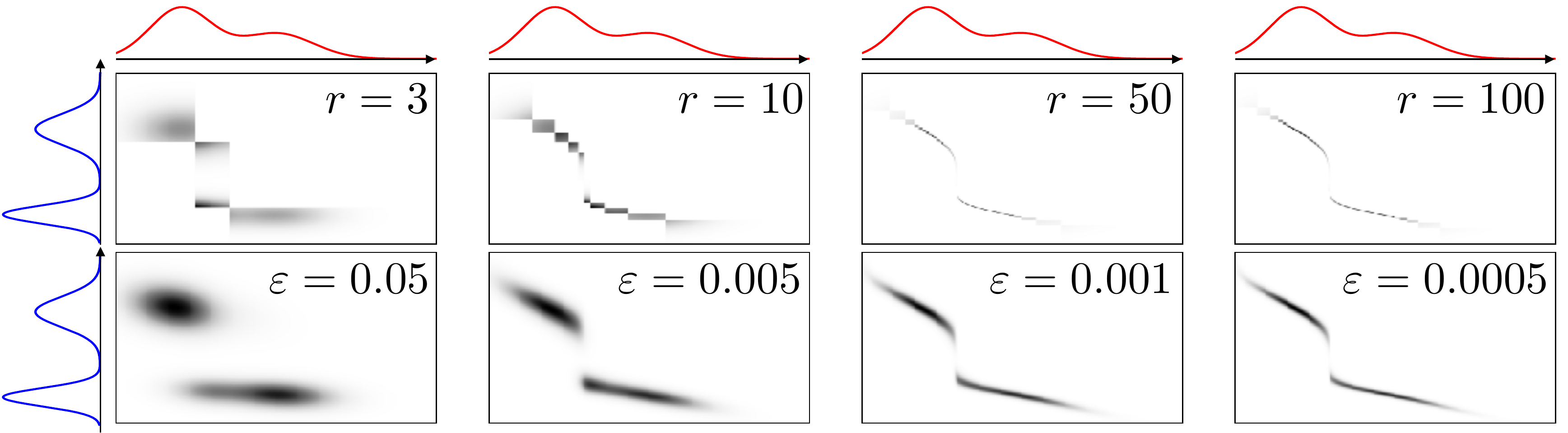}
\caption{Two Gaussian mixture densities evaluated on $n=200$ and $m=220$ sized grids in 1D, displayed as blue/red curves. Between them, $n\times m$ optimal coupling matrices obtained by our proposed low-rank OT method for varying rank constraint values $r$ (in increasing order, top row) and the Sinkhorn algorithm, for various $\varepsilon$ (in decreasing order, bottom row). The ground cost is the 1.5-norm.}\label{fig-coupling}
\end{figure*}

\section{Introduction}
By providing a simple and comprehensive framework to compare probability distributions, optimal transport (OT) theory has inspired many developments in machine learning~\citep{Peyre2019computational}. A flurry of works have recently connected it to other trending topics, such as normalizing flows or convex neural networks~\cite{pmlr-v119-makkuva20a,korotin2021continuous,pmlr-v119-tong20a}, while the scope of its applications has now reached several fields of science such as single-cell biology~\cite{schiebinger2019optimal,yang2020predicting}, imaging~\cite{schmitz2018wasserstein,heitz2020ground} or neuroscience~\cite{janati2020multi,koundal2020optimal}.

\textbf{Challenges when computing OT.} Solving optimal transport problems at scale poses, however, formidable challenges. The most obvious among them is computational: Instantiating the \citet{Kantorovich42} problem on discrete measures of size $n$ can be solved with a linear program (LP) of complexity $O(n^3\log n)$. A second and equally important challenge lies in the statistical performance of using that LP to estimate OT between densities: the LP solution between i.i.d samples converges very slowly to that between densities~\cite{fournier2015rate}. It is now increasingly clear that regularizing OT in some way or another is the only way to mitigate these two issues~\cite{genevay2018sample,chizat2020faster,CLASON2021124432}. A popular approach consists in penalizing the OT problem with a strongly convex function of the coupling~\cite{cuturi2013sinkhorn,dessein2018regularized}. We explore in this work an alternative, and more direct approach to add regularity: we restrict, instead, the set of feasible couplings to have a small nonnegative rank.

\textbf{Low-Rank Kernel Factorization.} 
Low-rank factorizations are not new to regularized OT. They have been used to speed-up the resolution of entropy regularized OT with the Sinkhorn algorithm, pending some approximations: Given a data-dependent $n\times m$ cost matrix $C$, the Sinkhorn iterations consist in matrix-vector products of the form $Kv$ or $K^Tu$ where $K\eqdef\exp(-C/\varepsilon)$ and $u, v$ are $n, m$- vectors. \citet{altschuler2018massively} and \citet{altschuler2020polynomialtime} have proposed to approximate the kernel $K$ with a product of thin rank $r$ matrices, $\widetilde{K}= AB^T$. Naturally, the ability to approximate $K$ with a low-rank $\widetilde{K}$ degrades as $\varepsilon$ decreases, making this approach valid only for sufficiently large $\varepsilon$. Thanks to this approximation, however, each Sinkhorn iteration is linear in $n$ or $m$, and the coupling outputted by the Sinkorn algorithm is of the form $\widetilde{P}=CD^T$ where $C=\Diag(u)A$, $D=\Diag(v)B$. This approximation results therefore in a \emph{low-rank} solution that is not, however, rigorously optimal for the original problem as defined by $K$ but rather that defined by $\widetilde{K}$. Similarly, \citet{scetbon2020linear} consider instead \textit{nonnegative low-rank} approximations for $K$ of the form $\widetilde{K}= QR^T$ where $Q,R>0$.x

\textbf{Low-Rank Couplings.} 
To our knowledge, only~\citet{forrow2018statistical} have used low rank considerations for couplings, rather than costs or kernels. Their work studies the case where the ground cost is the squared Euclidean distance. They obtain for that cost a proxy for rank-constrained OT problems using 2-Wasserstein barycenters~\cite{agueh2011barycenters}. Their algorithm blends those in~\citep{CuturiBarycenter,2015-benamou-cisc} and results in an intuitive mass transfer plan that goes through a small number of $r$ points, where $r$ is the coupling's nonnegative rank.

\textbf{Our Contributions.} In this work, we tackle directly the low-rank problem formulated by \cite{forrow2018statistical} but make no assumption on the cost matrix; we address instead the low-rank OT problem in its full generality. We consider couplings $P=Q\Diag(1/g)R^T$ decomposed as the product of two sub-couplings $Q,R,$ with common right marginal $g$, and left-marginal given by those of $P$ on each side. Each of these sub-couplings minimizes a transport cost that involves the original cost matrix $C$ and the other sub-coupling. We handle this problem by optimizing jointly on $Q$, $R$ and $g$ using a mirror-descent approach. We prove the non-asymptotic stationary convergence of this approach. In addition, we show that the time complexity of our algorithm can become linear when exploiting low rank assumptions on the \textit{cost} (not the kernel) involved in the OT problem.

\textbf{Differences with previous work.}
Our approach borrows ideas from~\cite{forrow2018statistical} but is generic as it applies to all ground costs. Our approach constrains the non-negative rank of the coupling solution $P$ by construction, rather than relying on a low rank approximation $\widetilde{K}$ for kernel $K=e^{-C/\varepsilon}$. This is a crucial point, because the ability to approximate $K$ with a low rank $\widetilde{K}$ significantly degrades as $\varepsilon$ decreases. By contrast, our approach applies to all ranks, small and large. Interestingly, we also show that a low-rank assumption on the cost matrix (not on the kernel) can also be leveraged, providing therefore a ``best of both worlds'' scenario in which both the \textit{coupling}'s and the \textit{cost}'s (not the kernel) low rank properties can be enforced and exploited. Finally, a useful parallel can be drawn between our approach and that of the vanilla Sinkhorn algorithm, in the sense that they propose different regularization schemes. Indeed, the (discrete) path of solutions obtained by our algorithm when varying $r$ between 1 and $\min(n,m)$ can be seen as an alternative to the entropic regularization path. Both paths contain at their extremes the original OT solution (maximal rank and minimal entropy) and the product of marginals (minimal rank and maximal entropy), as illustrated in Fig.~\ref{fig-coupling}.

\section{Discrete Optimal Transport}
\textbf{OT as a linear program.} Let $a$ and $b$ be two histograms in $\Delta_n, \Delta_m$, the probability simplices of respective size $n, m$. Assuming $a>0$ and $b>0$, set $X \eqdef (x_1,\dots,x_n)$ and $Y \eqdef (y_1,\dots,y_m)$ two families of points taken each within arbitrary sets, and define discrete distributions $\mu \eqdef \sum_{i=1}^n a_i \delta_{x_i}$ and $\nu \eqdef \sum_{j=1}^m b_j \delta_{y_j}$. The set of couplings with marginals $a, b$ is:
\begin{align*}
    \Pi_{a,b} \eqdef \{P\in\mathbb{R}_{+}^{n\times m} \text{ s.t. } P\mathbf{1}_m=a, P^{T}\mathbf{1}_n=b\}\,.
\end{align*}
Given a cost function $c$ defined on pairs of points in $X, Y$ and writing $C \eqdef [c(x_i,y_j)]_{i,j}$ its associated matrix, the optimal transport (OT) problem can be written as follows:
\begin{align}
\label{eq-OT}
    \text{OT}(\mu,\nu)  \eqdef 
    \min_{P\in\Pi_{a,b}}\langle C,P\rangle\, .
\end{align}
\textbf{Entropic regularization.} Several works have shown recently~\citep{genevay2018sample,chizat2020faster} that when $X$ and $Y$ are sampled from a continuous space, it is preferable to regularize~\eqref{eq-OT} using, for instance, an entropic regularizer~\cite{cuturi2013sinkhorn} to achieve both better computational and statistical efficiency,
\begin{align}
\label{eq-ROT}
    \text{OT}_{\varepsilon}(\mu,\nu) \eqdef \min_{P\in\Pi_{a,b}}\langle C,P\rangle -\varepsilon H(P)\,.
\end{align}
where $\varepsilon\geq 0$ and $H$ is the Shannon entropy defined as $H(P) \eqdef -\sum_{ij} P_{ij} (\log P_{ij}-1)$. If $\varepsilon$ goes to 0, one recovers the classical OT problem and for any $\varepsilon>0$, Eq.~(\ref{eq-ROT}) becomes $\varepsilon$-strongly convex on $\Pi_{a,b}$ and admits a unique solution
$P_{\varepsilon}$, of the form
\begin{equation}
\label{eq:sol-reg-ot}
\exists u_{\varepsilon}\in\mathbb{R}^{n}_{+}, v_{\varepsilon} \in \mathbb{R}^{m}_{+} \text{ s.t. }
P_{\varepsilon} = \text{diag}(u_\varepsilon) K \text{diag}(v_\varepsilon)
\end{equation}
where $K \eqdef \exp(-C/\varepsilon)$.
\citet{cuturi2013sinkhorn} shows that the scaling vectors $u_\varepsilon$ and $v_\varepsilon$ can be obtained efficiently thanks to the Sinkhorn algorithm (see Alg.~\ref{alg-sink}, where $\odot$ and $/$ denote entry-wise operation). Each iteration can be performed in $\mathcal{O}(nm)$ algebraic operations as it involves only matrix-vector products. The number of Sinkhorn iterations needed to converge to a precision $\delta$ (monitored by the difference between the column-sum of $\text{diag}(u)\mathbf{K}\text{diag}(v)$ and $b$) is controlled by the scale of elements in $C$ relative to $\varepsilon$~\cite{franklin1989scaling}. That convergence deteriorates with smaller $\varepsilon$, as studied in more detail by~\cite{altschuler2017near,pmlr-v80-dvurechensky18a}. 
\begin{algorithm}[H]
\SetAlgoLined
\textbf{Inputs:} $K,a,b,\delta, u$\\
\Repeat{$\|u\odot Kv - a\|_1 +\|v\odot K^T u - b\|_1<\delta$}{
    $v\gets b/K^T u,\;u\gets a/Kv$
  }
\KwResult{$u,v$}
\caption{$\text{Sinkhorn}(K,a,b,\delta)$ \label{alg-sink}}
\end{algorithm}

\textbf{Mirror descent and $\varepsilon$ schedule.} A possible interpretation of the entropic regularization in the OT problem is that it can be seen as the $k_{\varepsilon}$-th update of a Mirror Descent (MD) algorithm applied to the objective~(\ref{eq-OT}) where $k_{\varepsilon}\geq 1$ depends on $\varepsilon$ and the gradient steps used in the MD. Several works have proposed such links between a gradual decrease in $\varepsilon$ to obtain a better approximation of the unregularized OT problem ~\cite{schmitzer2019stabilized,pmlr-v97-lin19a,xie2020fast}. More precisely, the MD algorithm associated to the Kullback–Leibler divergence (KL) applied to the objective~(\ref{eq-OT}) makes for all $k\geq 0$ the following update:
\begin{align}
\label{eq-update-MD-OT}
    Q^{k+1} \eqdef \argmin_{Q\in\Pi_{a,b}} \langle C,Q\rangle  +\frac{1}{\gamma_k}\text{KL}(Q,Q_k)
\end{align}
where $(\gamma_k)_{k\geq 0}$ is a sequence of positive real numbers, $Q_0\in\Pi_{a,b}$ is an initial point and KL is the Kullback–Leibler divergence defined as $\text{KL}(P,Q) \eqdef \sum_{i,j}P_{i,j}(\log(P_{i,j}/Q_{i,j})-1)$. If $Q_0 \eqdef ab^T$, then one obtains that for all $k\geq 0$, updating the coupling  according to Eq.~(\ref{eq-update-MD-OT}) is the same as solving 
\begin{equation*}
    Q^{k+1} \eqdef \argmin_{Q\in\Pi_{a,b}} \langle C,Q\rangle  -\varepsilon_kH(Q)
\end{equation*}
where $\varepsilon_{k} \eqdef (\sum_{j=0}^{k}\gamma_j)^{-1}$.
Therefore the MD algorithm applied to~(\ref{eq-OT}) produces the sequence $(P_{\varepsilon_k})_{k\geq 0}$ of optimal couplings according to the objective~(\ref{eq-ROT}). We show next that this viewpoint can be applied when one adds also some structures to the couplings considered in the OT problem~(\ref{eq-OT}), leading to a new regularized approach.

\section{Nonnegative Factorization of the Optimal Coupling}

Here we aim at regularizing the OT problem by decomposing the couplings involved into a product of two low-rank couplings. 
We introduce the associated non-convex problem and develop a mirror-descent algorithm which operates by solving a succession of convex programs.

\subsection{Low Rank and Factored Couplings}

We introduce low rank couplings and explain how they can be parameterized as factored couplings.

\begin{defn}
\label{nonnegative-def}
Given $M\in\mathbb{R}^{n\times m}$, the nonnegative rank of $M$ is the smallest number of nonnegative rank-one matrices into which the matrix can be decomposed additively:
\begin{equation*}
\rank_{+}(M) \eqdef \min\!\left\{q|M=\sum_{i=1}^q R_i, \forall i, \rank(R_i)=1, R_i\geq 0\right\}.
\end{equation*}
\end{defn}
Let $r\geq 1$, and let us denote 
$$\Pi_{a,b}(r) \eqdef \{P\in\Pi_{a,b}, \rank_{+}(P)\leq r\}.$$
From Definition~\ref{nonnegative-def}, one has
\begin{align*}
\Pi_{a,b}(r)=
 \begin{aligned}[t]
  \Big\{
  &\sum_{i=1}^r g_i q_i r_i^T
  \text{ s.t. } \forall~i~q_i\in\Delta_n,~r_i\in\Delta_m,\\
  &g\in\Delta_r,~\sum_{i=1}^r g_i q_i =a \text{ and }
  \sum_{i=1}^r g_i r_i=b~
  \Big\}
 \end{aligned}
 \end{align*}
from which we deduce directly that $\Pi_{a,b}(r)$ is compact. Moreover for $g\in\Delta_r^{*} \eqdef \{h\in\Delta_r\text{ s.t. } \forall i~h_i>0\}$, we write
\begin{align*}
\Pi_{a,g,b} \eqdef 
 \begin{aligned}[t]
  \Big\{
  &P\in\mathbb{R}_{+}^{n\times m},  P=Q \Diag(1/g)R^T,\\
 & Q\in\Pi_{a,g}, \text{ and  } R\in\Pi_{b,g} 
  \Big\}.
 \end{aligned}
 \end{align*}
Note that $\Pi_{a,g,b}$ is compact and a subset of $\Pi_{a,b}(r)$ since for all $P\in\Pi_{a,g,b}$, $P\in\Pi_{a,b}$ and one has $\rank(P)\leq \rank_{+}(P)\leq r$. Moreover, for any $P\in\Pi_{a,b}$ such that $\rank_{+}(P)\leq r$, there exists $g\in\Delta_r^{*}$, $Q\in\Pi_{a,g}$ and $R\in\Pi_{b,g}$ such that $P=Q \Diag(1/g)R^T$~\cite{COHEN1993149}. Therefore
\begin{align}
\label{eq-reformulation-constraint}
    \bigcup\limits_{g\in\Delta_r^{*}}\Pi_{a,g,b}=\Pi_{a,b}(r).
\end{align}

We exploit next this identity to build an efficient algorithm in order to solve the optimal transport problem under low nonnegative rank constraints.

\subsection{The Low-rank OT Problem (LOT)}

The problem of interest in this work is:
\begin{align}
\label{eq-LOT}
    \text{LOT}_r(\mu,\nu) \eqdef \min_{P\in\Pi_{a,b}(r)}\langle C,P\rangle.
\end{align}
Here the minimum is always attained as $\Pi_{a,b}(r)$ is compact and the objective is continuous. Thanks to~\eqref{eq-reformulation-constraint}, problem~\eqref{eq-LOT} is equivalent to
\begin{align}
\label{eq-LOT-reformulated}
 \min_{(Q,R,g)\in\mathcal{C}(a,b,r)} \langle C,Q \Diag(1/g)R^T\rangle
\end{align}
where $\mathcal{C}(a,b,r) \eqdef \mathcal{C}_1(a,b,r)\cap \mathcal{C}_2(r)$, with
\begin{align*}
\mathcal{C}_1(a,b,r) \eqdef 
 \begin{aligned}[t]
  \Big\{
  &(Q,R,g)\in\mathbb{R}_{+}^{n\times r}\times\mathbb{R}_{+}^{m\times r}\times(\mathbb{R}_{+}^{*})^{r} \\
    & \text{ s.t. }  Q\mathbf{1}_r=a, R\mathbf{1}_r=b
  \Big\}
 \end{aligned}
 \end{align*}
 and
\begin{align*}
\mathcal{C}_2(r) \eqdef 
 \begin{aligned}[t]
  \Big\{
  &(Q,R,g)\in\mathbb{R}_{+}^{n\times r}\times\mathbb{R}_{+}^{m\times r}\times\mathbb{R}^r_{+} \\
    & \text{ s.t. }  Q^T\mathbf{1}_n=R^T\mathbf{1}_m=g 
  \Big\}.
 \end{aligned}
 \end{align*} 
In the following, we also consider regularized version of the problem~(\ref{eq-LOT-reformulated}) by adding an entropic term to the objective which leads for all $\varepsilon\geq 0$ to the following problem
\begin{multline}
\begin{aligned}
\label{eq-LOT-reformulated-ent}
\text{LOT}_{r,\varepsilon}(\mu,\nu)  \eqdef \inf_{(Q,R,g)\in\mathcal{C}(a,b,r)} \langle C, Q \Diag(1/g)R^T\rangle\\
  -\varepsilon H((Q,R,g)).
\end{aligned}
\end{multline}
Here the entropy of $(Q,R,g)$ is to be understood as that of the values of the three respective entropies evaluated for each term. We will see that adding an entropic term to the objective allows to stabilize the MD scheme employed to solve~\eqref{eq-LOT}. For all $\varepsilon\geq 0$, the objective function defined in~(\ref{eq-LOT-reformulated-ent}) is lower semi-continuous, and admits therefore a minimum in $\overline{\mathcal{C}_1(a,b,r)}\cap\mathcal{C}_2(r)$ where $\overline{\mathcal{C}_1(a,b,r)}$ is the closure of $\mathcal{C}_1(a,b,r)$. However, the existence of a solution for problem~(\ref{eq-LOT-reformulated-ent}) requires more care, as shown in the following proposition.

\begin{prop}
\label{prop:existence-min}
If $\varepsilon=0$ then the infimum of (\ref{eq-LOT-reformulated-ent}) is always attained. If $\varepsilon>0$, then if $r=1$, the infimum of (\ref{eq-LOT-reformulated-ent}) is attained and for $r\geq 2$, problem (\ref{eq-LOT-reformulated-ent}) admits a minimum if $\text{LOT}_{r,\varepsilon}(\mu,\nu)<\text{LOT}_{r-1,\varepsilon}(\mu,\nu)$.
\end{prop}

\paragraph{Stabilized Formulation using Lower Bounds} In order to ensure stability of the mirror descent method, and enable its theoretical analysis, we introduce a lower bound $\alpha$ on the weight vector $g$.

Let us assume in the following that we consider $(r,\varepsilon)$ satisfying the conditions of Proposition~\ref{prop:existence-min}. In particular if $\varepsilon=0$, $r$ can be arbitrarily chosen and we recover the problem defined in~(\ref{eq-LOT}). Under this assumption, there exists  $(Q^{*}_{\varepsilon},R^{*}_{\varepsilon},g_{\varepsilon}^{*})\in\mathcal{C}_1(a,b,r)\cap \mathcal{C}_2(r)$ solution of Eq.~(\ref{eq-LOT-reformulated-ent}) from which follows the existence of $\frac{1}{r}\geq \alpha^{*}>0$, such that $g_{\varepsilon}^{*}\geq \alpha^{*}$ coordinate-wise. Let us now define for any $\frac{1}{r}\geq \alpha> 0$, the following set 
\begin{align*}
\mathcal{C}_1(a,b,r,\alpha) \eqdef 
 \begin{aligned}[t]
  \Big\{
  &(Q,R,g)\in\mathbb{R}_{+}^{n\times r}\times\mathbb{R}_{+}^{m\times r}\times\mathbb{R}_{+}^{r} \\
    & \text{ s.t. }  Q\mathbf{1}_r=a, R\mathbf{1}_r=b,~g\geq \alpha
  \Big\}.
 \end{aligned}
 \end{align*}
Then if $\alpha$ is sufficiently small (i.e. $\alpha\leq\alpha^{*}$) we have that the problem~(\ref{eq-LOT-reformulated-ent}) is equivalent to
\begin{multline}
\label{eq-LOT-reformulated-alpha}
\text{LOT}_{r,\varepsilon,\alpha}(\mu,\nu)= \min_{(Q,R,g)\in\mathcal{C}(a,b,r,\alpha)} \langle C,Q \Diag(1/g)R^T\rangle\\
 -\varepsilon H((Q,R,g)),
\end{multline}
where $\mathcal{C}(a,b,r,\alpha) \eqdef \mathcal{C}_1(a,b,r,\alpha) \cap \mathcal{C}_2(r)$. Note that for any $\frac{1}{r}\geq \alpha>0$, the set of constraints is not empty, compact and the minimum always exists. 

\subsection{Mirror Descent Optimization Scheme}

\paragraph{Mirror descent outer loop.} 

We propose to use a Mirror Descent scheme with a \text{KL} divergence to solve Eq.~(\ref{eq-LOT-reformulated-alpha}). It leads, for all $k\geq 0$, to the following updates which necessitate the solution of a convex problem at each step 
\begin{equation}
\label{eq-barycenter-LOT-alpha}  
 (Q_{k+1},R_{k+1},g_{k+1}) \eqdef \!\! \argmin_{\bm{\zeta} \in\mathcal{C}(a,b,r,\alpha)} \!\! \text{KL}(\bm{\zeta},\bm{\xi}_k)
\end{equation}
where $(Q_0,R_0,g_0)\in\mathcal{C}(a,b,r,\alpha)$ is an initial point such that $Q_0>0$ and $R_0>0$,
$\bm{\xi}_k \eqdef (\xi_{k}^{(1)},\xi_{k}^{(2)},\xi_{k}^{(3)})$, $\xi_{k}^{(1)} \eqdef \exp(-\gamma_kCR_k \Diag(1/g_k)- (\gamma_k\varepsilon-1)\log(Q_k))$, $\xi_{k}^{(2)} \eqdef \exp(-\gamma_kC^TQ_k \Diag(1/g_k)- (\gamma_k\varepsilon-1)\log(R_k))$,
$\xi_{k}^{(3)} \eqdef \exp(\gamma_k\omega_k/g_k^2- (\gamma_k\varepsilon-1)\log(g_k))$ with  $[\omega_k]_i \eqdef [Q_k^TCR_k]_{i,i}$ for all $i\in\{1,\dots,r\}$ and $(\gamma_k)_{k\geq 0}$ is a sequence of positive step sizes. Note that for all $k\geq 0$, $(Q_{k},R_{k},g_{k})$ live in $(\mathbb{R}_{+}^{*})^{n\times r}\times (\mathbb{R}_{+}^{*})^{m\times r} \times (\mathbb{R}_{+}^{*})^{r}$, and therefore $\bm{\xi}_k$ is well defined and lives also in $(\mathbb{R}_{+}^{*})^{n\times r}\times (\mathbb{R}_{+}^{*})^{m\times r} \times (\mathbb{R}_{+}^{*})^{r}$.

\paragraph{Dykstra's inner loop.} In order to solve Eq.~(\ref{eq-barycenter-LOT-alpha}), we use the Dykstra's Algorithm~\cite{dykstra1983algorithm}. Given a closed convex set $\mathcal{C}\subset \mathbb{R}_{+}^{n\times r}\times \mathbb{R}_{+}^{m\times r} \times \mathbb{R}_{+}^{r}$, we denote for all $\bm{\xi}\in (\mathbb{R}_{+}^{*})^{n\times r}\times (\mathbb{R}_{+}^{*})^{m\times r} \times (\mathbb{R}_{+}^{*})^{r}$ the projection according to the Kullback-Leibler divergence as
\begin{align*}
    \mathcal{P}_{\mathcal{C}}^{\text{KL}}(\bm{\xi}) \eqdef \argmin_{\bm{\zeta}\in\mathcal{C}}\text{KL}(\bm{\zeta},\bm{\xi}).
\end{align*}
Starting from $\bm{\zeta}_0 \eqdef \bm{\xi}$ and $\bm{q}_{0}=\bm{q}_{-1}=(\mathbf{1},\mathbf{1},\mathbf{1})\in \mathbb{R}_{+}^{n\times r}\times \mathbb{R}_{+}^{m\times r} \times \mathbb{R}_{+}^{r}$, the Dykstra's Algorithm consists in computing for all $j\geq 0$, 
\begin{align*}
    \bm{\zeta}_{2j+1} &= \mathcal{P}_{\mathcal{C}_1(a,b,r,\alpha)}^{\text{KL}}(\bm{\zeta}_{2j}\odot \bm{q}_{2j-1})\\
    \bm{q}_{2j+1}& =  \bm{q}_{2j-1}\odot \frac{\bm{\zeta}_{2j}}{\bm{\zeta}_{2j+1}} \\
    \bm{\zeta}_{2j+2} &= \mathcal{P}_{\mathcal{C}_2(r)}^{\text{KL}}(\bm{\zeta}_{2j+1}\odot \bm{q}_{2j})\\
    \bm{q}_{2j+2}& =  \bm{q}_{2j} \odot \frac{\bm{\zeta}_{2j+1}}{\bm{\zeta}_{2j+2}}.
\end{align*}
As $\mathcal{C}_1(a,b,r,\alpha)$ and $\mathcal{C}_2(r)$ are closed convex subspaces and $\bm{\xi}\in (\mathbb{R}_{+}^{*})^{n\times r}\times (\mathbb{R}_{+}^{*})^{m\times r} \times (\mathbb{R}_{+}^{*})^{r}$, one can show that $(\bm{\zeta}_{j})_{j\geq 0}$ converges towards the unique solution of Eq.~(\ref{eq-barycenter-LOT-alpha}),~\cite{bauschke2000dykstras}. The following propositions detail how to compute the relevant projections involved in the Dykstra's Algorithm.

\begin{prop}
\label{prop:proj-C1}
For $\tilde{\bm{\xi}} \eqdef (\tilde{Q},\tilde{R},\tilde{g})\in (\mathbb{R}_{+}^{*})^{n\times r}\times (\mathbb{R}_{+}^{*})^{n\times r} \times (\mathbb{R}_{+}^{*})^{r}$, one has, denoting $\hat g \eqdef \max(\tilde{g},\alpha)$
\begin{align*}
    \mathcal{P}_{\mathcal{C}_1(a,b,r,\alpha)}^{\text{KL}}(\tilde{\bm{\xi}})=
    \left( \Diag\left(\frac{a}{\tilde{Q}\mathbf{1}_r}\right)\tilde{Q}, \Diag\left(\frac{b}{\tilde{R}\mathbf{1}_r}\right)\tilde{R},\hat g\right).
\end{align*}
\end{prop}
Let us now show the solution of the projection on $\mathcal{C}_2(r)$.
\begin{prop}
\label{prop:proj-C2}
For $\tilde{\bm{\xi}} \eqdef (\tilde{Q},\tilde{R},\tilde{g})\in (\mathbb{R}_{+}^{*})^{n\times r}\times (\mathbb{R}_{+}^{*})^{n\times r} \times (\mathbb{R}_{+}^{*})^{r}$, the projection $(Q,R,g)=\mathcal{P}_{\mathcal{C}_2(r)}^{\text{KL}}(\tilde{\bm{\xi}})$ satisfies 
\begin{align*}
Q&=\tilde{Q} \Diag(g/\tilde{Q}^T\mathbf{1}_n),\quad
R = \tilde{R} \Diag(g/\tilde{R}^T\mathbf{1}_m)\\
g &= (\tilde{g}\odot \tilde{Q}^T\mathbf{1}_n \odot \tilde{R}^T\mathbf{1}_m)^{1/3}.
\end{align*}
\end{prop}

\paragraph{Efficient computation of the updates.} The projection obtained in Proposition~\ref{prop:proj-C1},~\ref{prop:proj-C2} lead to simple updates of the couplings. Indeed, starting with $\bm\zeta_0 \eqdef \bm{\xi}=(\xi^{(1)},\xi^{(2)},\xi^{(3)})$ the Dysktra's Algorithm applied to our problem~\eqref{eq-barycenter-LOT-alpha} needs only to compute scaling vectors as presented in Alg.~\ref{alg-Dykstra}. We have denoted $p_1 \eqdef a$ and $p_2 \eqdef b$ to simplify the notations. See Appendix~\ref{sec-Dykstra} for more details.
\begin{algorithm}[H]
\SetAlgoLined
\textbf{Inputs:} $\xi^{(1)},\xi^{(2)},\tilde{g} \eqdef \xi^{(3)},p_1,p_2,\alpha,\delta,q^{(3)}_1=q^{(3)}_2=\mathbf{1}_r,\forall i\in\{1,2\},~ \tilde{v}^{(i)}=\mathbf{1}_r, q^{(i)}=\mathbf{1}_r$\\
\Repeat{$\sum_{i=1}^2\|u^{(i)}\odot \xi^{(i)}v^{(i)} - p_i\|_1 <\delta$}{
    $u^{(i)}\gets p_i/\xi^{(i)}\tilde{v}^{(i)}~\forall i\in\{1,2\},\\
    g\gets \max(\alpha,\tilde{g}\odot q^{(3)}_1),~
    q^{(3)}_1\gets (\tilde{g}\odot q^{(3)}_1)/ g,~\tilde{g}\gets g,\\
    g\gets (\tilde{g}\odot q^{(3)}_2)^{1/3} \prod_{i=1}^2 (v^{(i)}\odot q^{(i)}\odot(\xi^{(i)})^Tu^{(i)})^{1/3},\\ v^{(i)}\gets g/(\xi^{(i)})^T u^{(i)} ~\forall i\in\{1,2\},\\
    q^{(i)}\gets (\tilde{v}^{(i)}\odot q^{(i)})/v^{(i)}~\forall i\in\{1,2\},~q^{(3)}_2 \gets (\tilde{g}\odot q^{(3)}_2)/g,\\
    \tilde{v}^{(i)}\gets v^{(i)}~\forall i\in\{1,2\},~\tilde{g}\gets g
    $
  }
$Q\gets  \Diag(u^{(1)})\xi_{k}^{(1)} \Diag(v^{(1)})$\\
$R\gets  \Diag(u^{(2)})\xi_{k}^{(2)} \Diag(v^{(2)})$\\
\textbf{Result:} $Q,R,g$
\caption{$\text{LR-Dykstra}((\xi^{(i)})_{1\leq i\leq 3},p_1,p_2,\alpha,\delta)$ \label{alg-Dykstra}}
\end{algorithm}

Let us now introduce the proposed MD algorithm applied to~(\ref{eq-LOT-reformulated-alpha}). By denoting $\mathcal{D}(\cdot)$ the operator extracting the diagonal of a square matrix we  obtain  Alg.~\ref{alg-MDLROT-alpha}.

\begin{algorithm}[H]
\SetAlgoLined
\textbf{Inputs:} $C,a,b,(\gamma_k)_{k\geq 0}, Q,R,g,\alpha,\delta$\\
\For{$k=1,\dots$}{
    $\xi^{(1)}\gets\exp(-\gamma_kCR \Diag(1/g)- (\gamma_k\varepsilon-1)\log(Q)),\\ \xi^{(2)}\gets\exp(-\gamma_kC^TQ \Diag(1/g)- (\gamma_k\varepsilon-1)\log(R)),\\
    \omega\gets \mathcal{D}(Q^TCR),\\
    \xi^{(3)}\gets\exp(\gamma_k\omega/g^2- (\gamma_k\varepsilon-1)\log(g)),\\
    Q,R,g\gets \text{LR-Dykstra}((\xi^{(i)})_{1\leq i\leq 3},a,b,\alpha,\delta)~(\text{Alg.~\ref{alg-Dykstra})}$
  }
\textbf{Result:} $\langle C,Q \Diag(1/g)R^T\rangle$
\caption{$\text{LOT}(C,a,b,r,\alpha,\delta)$ \label{alg-MDLROT-alpha}}
\end{algorithm}

\paragraph{Computational Cost.} Note that $(\xi^{(i)})_{1\leq i\leq 3}$ considered in Alg.~\ref{alg-MDLROT-alpha} live in $\mathbb{R}_{+}^{n\times r}\times\mathbb{R}_{+}^{m\times r} \times\mathbb{R}_{+}^{r} $ and therefore given those matrices, each iteration of Alg.~\ref{alg-Dykstra} requires $\mathcal{O}((n+m)r)$  algebraic operations, since it involves only matrix/vector multiplications of the form $\xi^{(i)}v_i$ and $(\xi^{(i)})^Tu_i$. However without any assumption on the cost matrix $C$, computing $(\xi^{(i)})_{1\leq i\leq 3}$ requires $\mathcal{O}(nmr)$ algebraic operations since $CR$ and $C^TQ$ must be evaluated. We show in \S\ref{sec-OT-linear} how to reduce the quadratic cost of computing $(\xi^{(i)})_{1\leq i\leq 3}$ to a linear cost with respect to the number of samples if one assumes that the considered \textit{cost} matrix can be factored, either exactly (ensured with a squared Euclidean distance cost) or approximately if that cost is a distance. Writing $N$ the number of iterations of the MD scheme and $T$ the number of iterations considered in Algorithm~{\ref{alg-Dykstra}} at each step of the MD, we end up with a total computational cost of $\mathcal{O}(NT(n+m)r + Nnmr)$.

\subsection{Convergence of the Mirror Descent} 

Even if the objective~(\ref{eq-LOT-reformulated-alpha}) is not convex in $(Q,R,g)$, we obtain the non-asymptotic stationary convergence of the MD algorithm in this setting. For that purpose we introduce a stronger convergence criterion than the one presented in~\cite{ghadimi2013minibatch} to obtain non-asymptotic stationary convergence of the MD scheme. Indeed let $F_{\varepsilon}$ be the objective function of the problem~(\ref{eq-LOT-reformulated-alpha}) defined on $\mathcal{C}(a,b,r,\alpha)$ and let us denotes for any $\gamma>0$ and $\bm{\xi}\in \mathcal{C}(a,b,r,\alpha)$ 
\begin{align*}
    \mathcal{G}_{\varepsilon,\alpha}(\bm{\xi},\gamma) \eqdef \argmin_{\bm{\zeta}\in \mathcal{C}(a,b,r,\alpha)}\{ \langle \nabla F_{\varepsilon}(\bm{\xi}),\bm{\zeta} \rangle +\frac{1}{\gamma} KL(\bm{\zeta},\bm{\xi}) \}.
\end{align*}
Then the criteron used in~\cite{ghadimi2013minibatch} to show the stationary convergence of the MD scheme is defined as the square norm of the following vector:
\begin{align*}
   P_{\mathcal{C}(a,b,r,\alpha)}(\bm{\xi},\gamma) \eqdef \frac{1}{\gamma}(\bm{\xi}-\mathcal{G}_{\varepsilon,\alpha}(\bm{\xi},\gamma)).
\end{align*}
This vector can be seen as a generalized projected gradient of $F_\varepsilon$ at $\bm{\xi}$. Indeed if $X = \mathbb{R}^d$ and by replacing the \emph{prox-function} $\text{KL}(u,x)$ by $\frac{1}{2}\Vert u-x\Vert_2^2$, we would have $P_{X}(x,\gamma)=\nabla F_\varepsilon(x)$. Here we consider instead the following criterion to establish convergence:
\begin{align*}
   \Delta_{\varepsilon,\alpha}(\bm{\xi},\gamma) \eqdef \frac{1}{\gamma^2}(\mathrm{KL}(\bm{\xi},\mathcal{G}_{\varepsilon,\alpha}(\bm{\xi},\gamma))+\mathrm{KL}(\mathcal{G}_{\varepsilon,\alpha}(\bm{\xi},\gamma),\bm{\xi})).
\end{align*}
Such criterion is in fact stronger than the one used in~\cite{ghadimi2013minibatch} as we have
\begin{align*}
    \Delta_{\varepsilon,\alpha}(\bm{\xi},\gamma)&=\frac{1}{\gamma^2}(\langle  \nabla h(\mathcal{G}_{\varepsilon,\alpha}(\bm{\xi},\gamma)) - \nabla h(\bm{\xi}),  \mathcal{G}_{\varepsilon,\alpha}(\bm{\xi},\gamma) - \bm{\xi}\rangle\\
    &\geq \frac{1}{2\gamma^2}  \Vert \mathcal{G}_{\varepsilon,\alpha}(\bm{\xi},\gamma) - \bm{\xi}\Vert_1^2\\
    &=\frac{1}{2}\Vert  P_{\mathcal{C}(a,b,r,\alpha)}(\bm{\xi},\gamma)\Vert_1^2
\end{align*}
where $h$ denotes the minus entropy function and the last inequality comes from the strong convexity of $h$ on $\mathcal{C}(a,b,r,\alpha)$.  

For any $\frac{1}{r}\geq \alpha> 0$, we show in the following proposition the non-asymptotic stationary convergence of the MD scheme applied to the problem~(\ref{eq-LOT-reformulated-alpha}). To prove this result, we show that for any $\varepsilon\geq 0$, the objective is smooth relatively to the negative entropy function~\cite{bauschke2017descent} and we extend the proof of~\cite{ghadimi2013minibatch} to this case. 
\begin{prop}
\label{prop:cvg-MD-Dykstra}
Let $\varepsilon\geq 0$, $\frac{1}{r}\geq \alpha> 0$ and $N\geq 1$. By denoting $$L_{\varepsilon,\alpha} \eqdef \sqrt{3\left(2\frac{\Vert C\Vert_2^2}{\alpha^4}+\left(\frac{\varepsilon+2\Vert C\Vert_2}{\alpha^3}\right)^2\right)}$$
and by considering a constant stepsize in the MD scheme~(\ref{eq-barycenter-LOT-alpha}) such that for all $k=1,\dots,N$ $\gamma_k=\frac{1}{2L_{\varepsilon,\alpha}}$, we obtain that
\begin{align*}
    \min_{1\leq k\leq N}\Delta_{\varepsilon,\alpha}((Q_k,R_k,g_k),\gamma_k)\leq \frac{4L_{\varepsilon,\alpha}  D_0}{N}.
\end{align*}
where $D_0 \eqdef  F_{\varepsilon}(Q_0,R_0,g_0)- \mathrm{LOT}_{r,\varepsilon,\alpha}$ is the distance of the initial value to the optimal one. 
\end{prop}
Thanks to Proposition~\ref{prop:cvg-MD-Dykstra}, for $\alpha$ sufficiently small (i.e. $\alpha\leq \alpha^{*}$), we have $\text{LOT}_{r,\varepsilon,\alpha}=\text{LOT}_{r,\varepsilon}$ and therefore we obtain a stationary point of~(\ref{eq-LOT-reformulated-ent}). In particular, if $\varepsilon=0$, the proposed algorithm converges towards a stationary point of~(\ref{eq-LOT}).

\begin{rmq}
We also propose an algorithm to directly solve~(\ref{eq-LOT-reformulated-ent}). The main difference is that the updates of the MD can be solved using the Iterative Bregman Projections (IBP) Algorithm. See Appendix~\ref{sec-IBP} for more details.
\end{rmq}

\begin{rmq}
For all $\varepsilon\geq 0$, the MD scheme implies that each iteration $k$ of our proposed algorithm outputs $(Q_k,R_k,g_k)\in\mathcal{C}_1(a,b,r,\alpha)\cap\mathcal{C}_2(r)$, and therefore the matrix obtained a each iteration $P_{k}^{\textit{LOT}}=Q_k\Diag(1/g_k)R_k^T$ is a coupling which sastifies the marginal constraints while in the Sinkhorn algorithm, the matrix defined at each iteration by $P_{k}^{\text{Sin}}=\Diag(u_k)K\Diag(v_k)$ becomes a coupling which satisfies the marginal constraints only at convergence.
\end{rmq}

In the following section, we aim at accelerating our method in order to obtain a linear time algorithm to solve~(\ref{eq-LOT-reformulated-ent}).


\begin{figure*}[!t]
\centering
\includegraphics[width=1\textwidth]{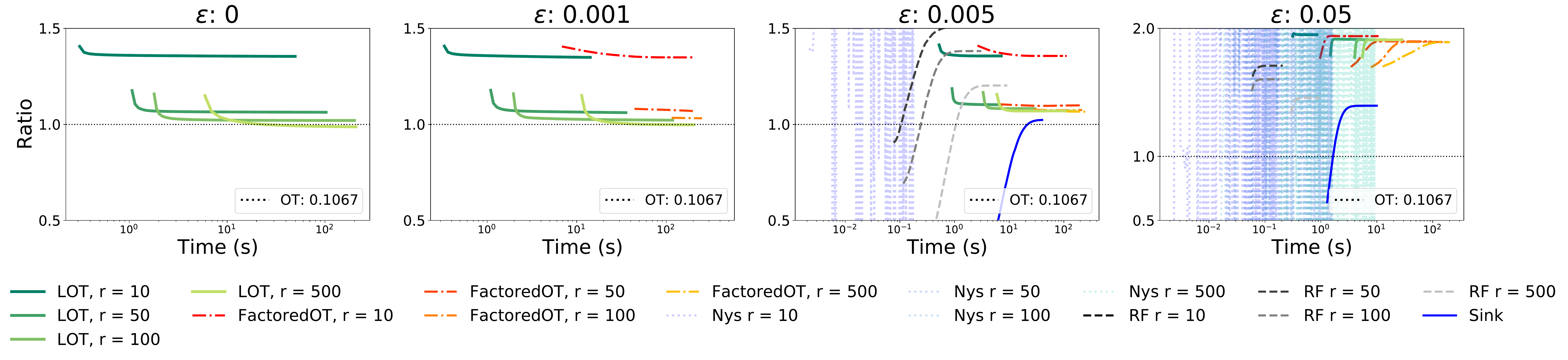}
\caption{In this experiment, we consider two Gaussian distributions evaluated on $n=m=5000$ in 2D. The first one has a mean of $(1,1)^T$ and identity covariance matrix $I_2$ while the other has 0 mean and covariance $0.1\times I_2$. The ground cost is the squared Euclidean distance. Note that for this cost, an exact low-rank factorization of the cost is available, and therefore all low-rank methods, including ours, have a linear time complexity. 
\emph{Left:} we show that when $\varepsilon=0$ our method is able to quickly obtain the exact OT by forcing the nonnegative rank of the coupling to be relatively small compared to the number of samples. Note that in this setting, all the other methods cannot be applied.  \emph{Middle left, middle right:} In these plots, we show that our method can obtain high accuracy for either estimate the true OT or its regularized version with order of magnitude faster than the other low-rank methods for any rank $r$. Moreover, our methods outperforms \textbf{Sin} in these regimes of small regularizations. Note that \textbf{Sin} does not converge for $\varepsilon=0.002$ as we do not consider its stabilized version using log-sum-exp function but rather its classical version which is less costly to compute. \emph{Right:} Here we change the scale of the $y$-axis of the plot. We see that the regime of the entropic regularizations for the Sinkhorn algorithm and our method differs. Indeed, the Sinkhorn algorithm has a larger range of $\varepsilon$ such that it provides an efficient approximation of the OT, whereas \textbf{LOT} is regularizing \textit{twice}, namely with respect to both rank \textit{and} entropy.
}\label{fig-LR-Square-Euclidean}
\vspace{-0.1cm}
\end{figure*}

\subsection{Linear time approximation of the Low-Rank Optimal Transport}
\label{sec-OT-linear}
Here we aim at obtaining the optimal solution of Eq.~(\ref{eq-LOT-reformulated-ent}) in linear time with respect to the number of samples. For that purpose let us introduce our main assumption on the cost matrix $C$.
\begin{assumption}
\label{assump-low-rank-sin}
Assume  that $C$ admits a low-rank factorization, that is there exists $A\in\mathbb{R}^{n\times d}$ and $B\in\mathbb{R}^{m\times d}$ such that $C = AB^{T}.$
\end{assumption}
From the Assumption~\ref{assump-low-rank-sin} one can in fact accelerate the computation in the iterations of the proposed Alg.~(\ref{alg-MDLROT-alpha}) and obtain a linear time algorithm with respect to the number of samples. Indeed recall that given $\bm{\xi}=(\xi^{(i)})_{1\leq i\leq 3}$, each iteration of the Dykstra's Alg.~(\ref{alg-Dykstra}) can be performed in linear time. Moreover, thanks to Assumption~\ref{assump-low-rank-sin}, the computation of $\bm{\xi}$, which requires to compute both $CR$ and $C^TQ$ can be performed in $\mathcal{O}((n+m)dr)$ algebraic operations and thus Alg.~(\ref{alg-MDLROT-alpha}) requires only a linear number of algebraic operations with respect to the number of samples at each iteration.

Let us now justify why the Assumption~\ref{assump-low-rank-sin} of a low-rank factorization for the cost matrix is well suited in the problem of computing the Optimal Transport.

\paragraph{Squared Euclidean Metric.} In the specific case where $C$ is a Square Euclidean distance matrix, it admits a low-rank decomposition. Indeed let $X \eqdef [x_1,\dots,x_n]\in\mathbb{R}^{d\times n}$, let $Y \eqdef [y_1,\dots,y_m]\in\mathbb{R}^{d\times m}$ and let $D \eqdef (\Vert x_i-y_j\Vert_2^2)_{i,j}$. Then by denoting $p=[\Vert x_1\Vert_2^2,\dots,\Vert x_n\Vert_2^2]^T\in\mathbb{R}^n$ and $q=[\Vert y_1\Vert_2^2,\dots,\Vert y_m\Vert_2^2]^T\in\mathbb{R}^m$ we can rewrite $D$ as the following:
\begin{align*}
    D=p\mathbf{1}_m^T + \mathbf{1}_nq^T - 2 X^T Y.
\end{align*}
Therefore by denoting $A=[p,\mathbf{1}_n,-2X^T]\in\mathbb{R}^{n\times (d+2)}$ and $B=[\mathbf{1}_m,q,Y^T]\in\mathbb{R}^{n\times (d+2)}$ we obtain that
\begin{align*}
    D=AB^T.
\end{align*}

\paragraph{General Case: Distance Matrix.} In the following we denote a distance matrix $D\in\mathbb{R}^{n\times m}$, any matrix such that there exists a metric space $(\mathcal{X},d)$, $\{x_i\}_{i=1}^n \in\mathcal{X}^n$ and $\{y_j\}_{j=1}^m\in \mathcal{X}^m$ which satisfy for all $i,j$, $D_{i,j}=d(x_i,y_j)$. In fact it is always possible to obtain a low-rank approximation of a distance matrix in linear time. In~\cite{bakshi2018sublinear,indyk2019sampleoptimal}, the authors proposed an algorithm such that for any distance matrix $D\in\mathbb{R}^{n\times m}$ and $\gamma>0$ it outputs matrices $M\in\mathbb{R}^{n\times d}$, $N\in\mathbb{R}^{m\times d}$ in $\mathcal{O}((m+n)\text{poly}(\frac{d}{\gamma}))$ algebraic operations such that with probability at least $0.99$ we have
\begin{align*}
    \Vert D - MN^T\Vert_F^2\leq \Vert D - D_d\Vert_F^2 +\gamma\Vert D\Vert_F^2
\end{align*}
where $D_d$ denotes the best rank-$d$ approximation to $D$. Therefore one can always obtain a low-rank factorization of a distance matrix in linear time with respect to the number of samples. See Appendix~\ref{sec-LR-Distance} for more details.

\section{Numerical Results}
We consider three problems in which we study the time-accuracy trade-off as well as the couplings obtained, by comparing our method with other low-rank methods, as well as Sinkhorn's algorithm. More precisely, we compare our proposed method, \textbf{LOT}, with the factored Optimal Transport~\cite{forrow2018statistical}, \textbf{FactoredOT}, the Nystrom-based method~\cite{altschuler2018massively}, \textbf{Nys},  the random features-based method~\cite{scetbon2020linear}, \textbf{RF} and the Sinkhorn algorithm~\cite{cuturi2013sinkhorn}, \textbf{Sin}. For \textbf{LOT}, and in all experiments, we set the lower bound on $g$ to $\alpha=10^{-5}$.

\paragraph{Time-accuracy Tradeoff} We consider two problems where the ground cost involved in the OT problem is either the \emph{squared Euclidean} distance or the \emph{Euclidean} distance.
In the first one, we consider measures supported on $n=5000$ points in $\mathbb{R}^2$, while the second we consider $n=10000$ samples in $\mathbb{R}^{2}$. The method proposed by~\cite{forrow2018statistical} can only be used with the squared Euclidean distance (2-Wasserstein) while ours works for any cost.  For all the low-ranks methods, we vary the ranks between 10 and 500. For all the randomized methods, we consider the mean over 10 runs to estimate the OT. 

In Fig.~\ref{fig-LR-Square-Euclidean},~\ref{fig-LR-Euclidean} we plot the ratio w.r.t. the (non-regularized) optimal transport cost defined as $\text{R}:= \langle C,\widetilde{P}\rangle / \langle C,P^{*}\rangle$ where $\widetilde{P}$ is the coupling obtained by the method considered and $P^{*}$ is the ground truth (we ensure this optimal cost is large enough to avoid spurious divisions by 0). We present the time-accuracy tradeoffs of the methods for different regularizations $\varepsilon$ and ranks $r$. We show that our method provides consistently a better approximation of the OT while being much faster than the other low-rank methods for various targeted rank values $r$. We also show that our method is able to approximate arbitrarily well the OT and so faster than the Sinkhorn algorithm thanks to the low-rank constraints. We compare the methods in the same setting but we increase the dimensionality of the problems considered and we observe similar results. See Appendix~\ref{sec-exp-add} for more details.

\begin{figure*}[!ht]
\centering
\includegraphics[width=1\textwidth]{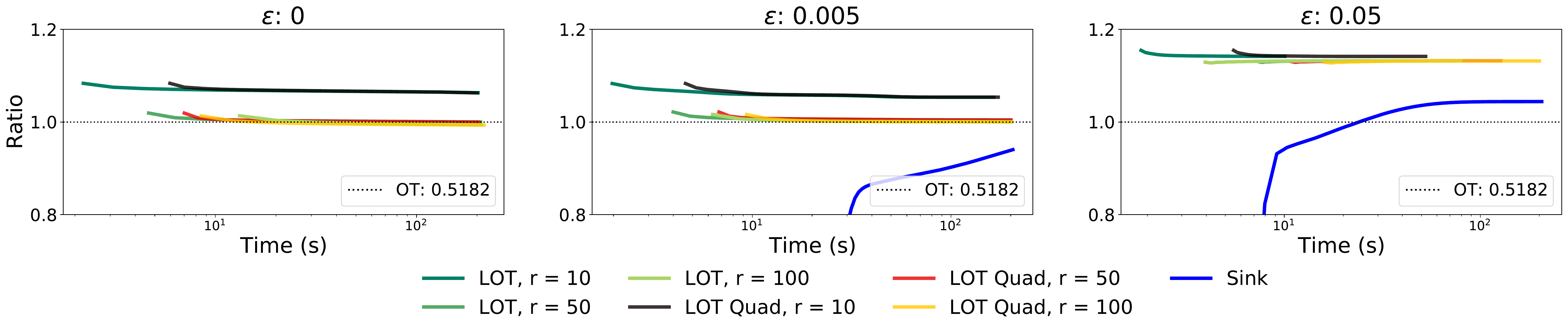}
    \caption{Here we consider two Gaussian mixture densities sampled with $n=m=10000$ points in 2D (See Appendix~\ref{sec-exp-add} for more details). The ground cost is the Euclidean distance. As this cost is a distance, we can apply our linear version of the algorithm and we denote \textbf{LOT Quad} to refer to its quadratic counterpart. We see that  \textbf{LOT} and \textbf{LOT Quad} provide similar results while \textbf{LOT} is faster.
    All kernel-based methods (\textbf{Nys}, \textbf{RF}) fail to converge in this setting. As in Fig.~\ref{fig-LR-Square-Euclidean}, we see that our method is able to approximate faster than \textbf{Sin} the true OT thanks to the low-rank constraint.
    }\label{fig-LR-Euclidean}
\vspace{-0.1cm}
\end{figure*}

\begin{figure}[!h]
\begin{center}
\includegraphics[width=0.44\textwidth,height=0.31\textwidth]{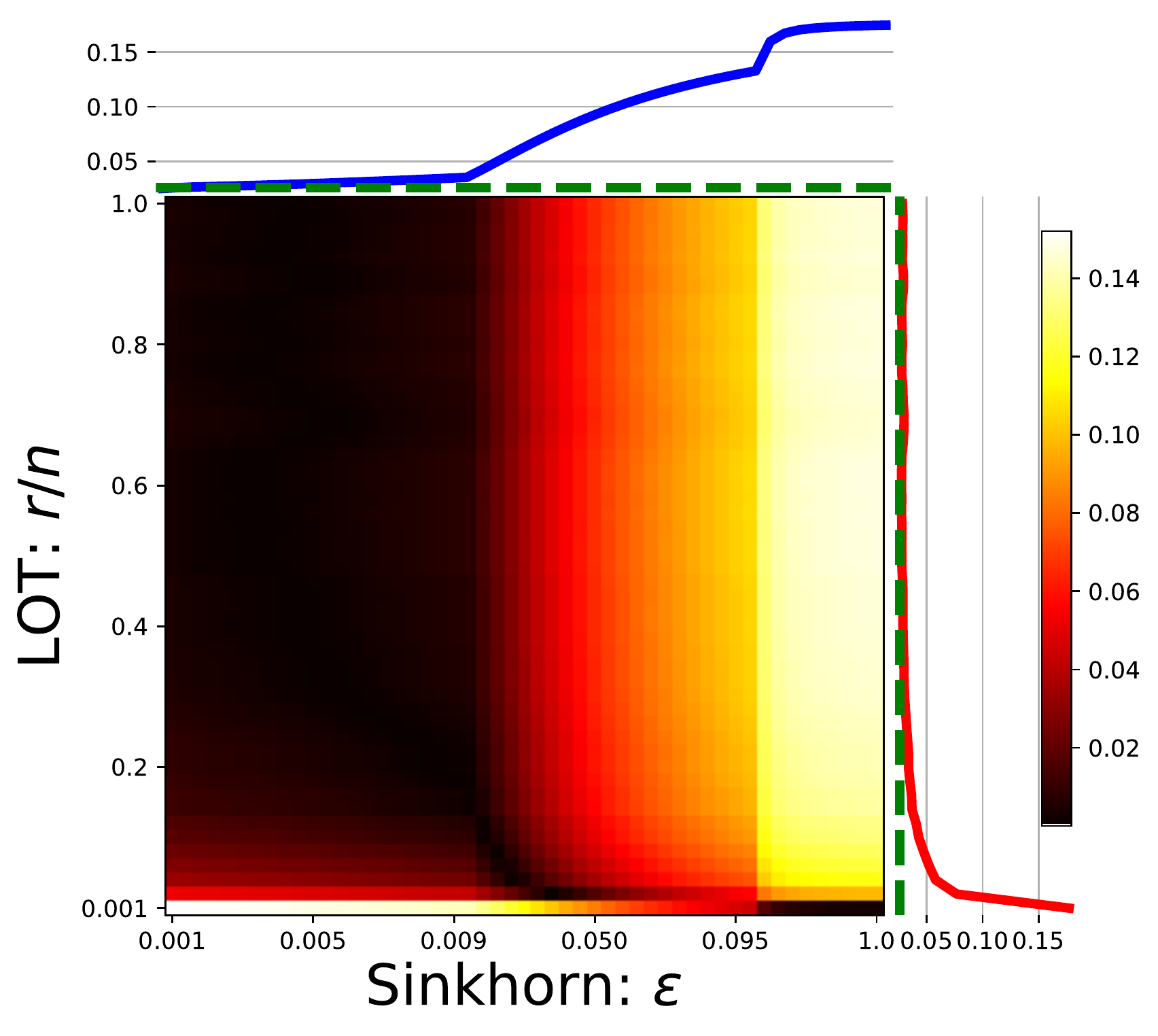}\\
\includegraphics[width=0.44\textwidth,height=0.31\textwidth]{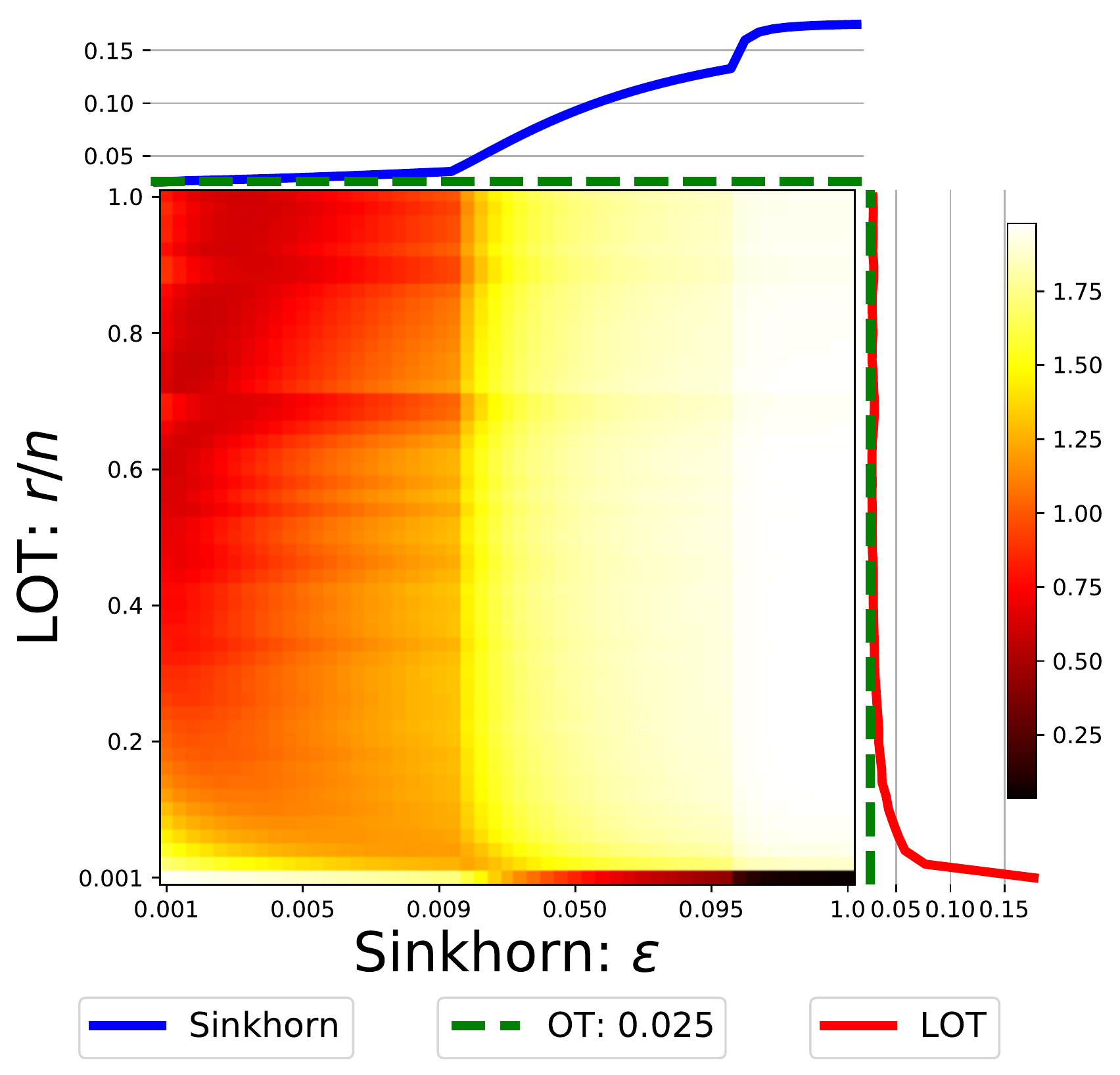} 
\caption{We illustrate in this plot the gaps between the OT objective (top) and the couplings (bottom) reached by \textbf{Sin} and \textbf{LOT} for varying regularization strengths. Measures were sampled on a complete graph obtained by sampling $2n=2000$ points from a 2-D standard normal distribution, the edge weights set to their squared Euclidean distances. The supports are obtained by randomly splitting the nodes of the graphs into two subsets of same size. We vary the entropic regularization $\varepsilon$ and the nonnegative rank $r$. We consider $\varepsilon$ in log-scale ranging from $0.001$ to $1$ and $r$ ranging from 1 to 1000, represented as a fraction of $n$. The blue (resp. red) curve stands for \textbf{Sin} (resp. \textbf{LOT}). We plot the absolute difference between the OT obtained (top) and $\ell_1$ distance between their respective couplings (bottom).}
\label{fig:couplings}
\end{center}
\vspace{-0.1cm}
\end{figure}

\begin{rmq}
Adding an entropic regularization in our objective allows to stabilize the MD scheme and therefore obtain faster convergence. Indeed if $\varepsilon>0$, then the number of iterations required to solve each iteration of the MD scheme~(\ref{eq-barycenter-LOT-alpha}) by Algorithm~(\ref{alg-Dykstra}) is monitored by $\varepsilon$ given a certain precision $\delta$ while in the case where $\varepsilon=0$, the number of iterations required for Algorithm~\ref{alg-Dykstra} to reach the precision $\delta$ increases as the number of iterations in the MD scheme increases. 
\end{rmq}

\paragraph{Comparison of the Couplings} Seeking to take a deeper look at the phenomenon highlighted in Fig.~\ref{fig-coupling}, we study differences in the regularization paths of \textbf{LOT} and \textbf{Sin}. We consider distributions supported on graphs of $n=1000$ nodes, endowed with the shortest path distance~\cite{bondy1976graph}. We consider $\textbf{LOT}$ with \textit{no} entropic regularization (i.e. $\varepsilon=0$ in Eq.~\eqref{eq-LOT-reformulated-alpha}) against \textbf{Sin} for various pairs of regularizers. Results are displayed in Fig.~\ref{fig:couplings}, where the discrete path of regularizations parameterized by the rank $r$ of \textbf{LOT} is compared with that obtained by \textbf{Sin} when varying $\varepsilon$. The gaps in ratio $\text{R}$ and couplings (in $\ell_1$) between the two methods are displayed. Both methods are able to approximate arbitrarily well the OT but offer two different paths to interpolate from the independent coupling $ab^T$ of rank 1 to the optimal one. More precisely, we see that the range of $\varepsilon$ for which the entropic OT provides an efficient approximation of the true coupling is very localized, while the rank $r$ needed for $\textbf{LOT}$ to obtain such approximation is wider. Moreover, we see that the decay of the ratio of \textbf{LOT} with respect to $r$ is faster than the decay of \textbf{Sin} w.r.t. $\varepsilon$.

\begin{rmq}
A comparative advantage of using the low-rank parameterization of OT over the Sinkhorn approach lies in the simple bounds that $r$ admits, between $1$ and $n$, and the fact that $r$ encodes directly, through an integer, a direct property of the resulting coupling. In that sense, the same value $r$ can be used across experiments that compare measures of various sizes and supports. By contrast, selecting a suitable regularization strength $\varepsilon$ in the Sinkhorn algorithm is usually challenging, as the parameter is continuous and its magnitude depends directly on the cost matrix values, making a common choice across experiments difficult.
\end{rmq}

\paragraph{Conclusion} We proposed a new approach to regularize the OT problem by restricting solutions to have a small non-negative rank. Our algorithm leverages both low-rank constraints and entropic smoothing. Our method can leverage the factorization of the ground cost (and \textit{not} that of the kernel usually associated to Sinkhorn) to propose a linear time complexity alternative to solve OT problems.

\paragraph{Acknowledgements}

The work of G. Peyré was supported by the European Research Council (ERC project NORIA) and by the French government under management of ANR as part of the ``Investissements d’avenir'' program (ANR19-P3IA-0001, PRAIRIE 3IA Institute).

\clearpage
\newpage
\bibliography{biblio}
\bibliographystyle{icml2021}

\clearpage
\appendix


%
%





%

%

\onecolumn
\onecolumn
\section*{Supplementary material}
In Sec.~\ref{sec-MD}, we introduce some important notions linked to the mirror-descent scheme. We also prove in this section a general result which states the non-asymptotic stationary convergence of the mirror-descent according to a specific criterion introcuded in this work. In Sec.~\ref{sec-Dykstra}, we detail the computation of the Dykstra's algorithm~\ref{alg-Dykstra} for which we have obtained a simple expression of the updates of the couplings. In Sec.~\ref{sec-proofs}, we provides all the proofs of the Propositions introduced in this work in the main text. In Sec~\ref{sec-LR-Distance}, we detail the algorithm presented in~\cite{indyk2019sampleoptimal}. In Sec.~\ref{sec-LR-fixed}, \ref{sec-LR-IBP}, we give two variants of our algorithm when either the marginal $g$ is fixed or when no lower bound is provided on the coordinates of $g$. In Sec.~\ref{sec-exp-add}, we provides more experiment to illustrate our method.


\section{Mirror Descent Algorithm}
\label{sec-MD}
Let $\mathcal{X}$ a closed convex subset in a Euclidean space $\mathbb{R}^q$, $f:\mathcal{X}\rightarrow{\mathbb{R}}$ continuously differentiable and let us consider the following problem
\begin{align}
\label{optim-problem}
    \min_{x\in\mathcal{X}}f(x).
\end{align}
Given a convex function $h:\mathcal{X}\rightarrow\mathbb{R}$ continuously differentiable, one can define the \emph{prox-function} associated to $h$ as 
\begin{align*}
    D_h(x,z):=h(x)-h(z)-\langle \nabla h(z),x-z\rangle.
\end{align*}
To solve Eq.~\eqref{optim-problem}, one can employ the mirror-descent (MD) algorithm. Given an initial point $x_0\in\mathcal{X}$ and a sequence of positive step-size $(\gamma_k)_{k\geq 0}$, the mirror-descent scheme associated to the \emph{prox-function} $D_h$ computes
\begin{align*}
    x_{k+1}=\argmin_{x\in\mathcal{X}} \langle\nabla f(x_k),x\rangle +\frac{1}{\gamma_k}D_h(x,x_k).
\end{align*}
In the following, we need to introduce two notions of relative strong convexity and relative smoothness in order to prove non-asymptotic stationary convergence of the MD scheme. 
\begin{defn*}[Relative smoothness.]
Let $L>0$ and $f$ continuously differentiable on $\mathcal{X}$. $f$ is said to be $L$-smooth relatively to $h$ if
\begin{align*}
    f(y)\leq f(x) +\langle \nabla f(x),y-x\rangle + L D_h(y,x)
\end{align*}
\end{defn*}

\begin{defn*}[Relative strong convexity]
Let $\alpha>0$ and $f$ continuously differentiable on $\mathcal{X}$. $f$ is said to be $\alpha$-strongly convex relatively to $h$ if
\begin{align*}
    f(y)\geq f(x) +\langle \nabla f(x),y-x\rangle + \alpha D_h(y,x)~\forall~x,y\in\mathcal{X}
\end{align*}
\end{defn*}
Note that $h$ is always 1-strongly convex relatively to $h$. Let us now prove a general result to show non-asymptotic stationary convergence of the MD scheme. For that purpose, we introduce for all $k\geq 0$ the following criterion to establish convergence:
\begin{align*}
   \Delta_k \eqdef \frac{1}{\gamma_k^2}(D_h(x_k,x_{k+1})+D_h(x_{k+1},x_k)).
\end{align*}
\begin{prop}
\label{prop-MD-general}
Let $N\geq 1$, $f$ continuously differentiable on $\mathcal{X}$ which is $L$-smooth relatively to $h$. By considering for all $k=1,\dots, N$, $\gamma_k=1/2L$, and by denoting  $D_0=f(x_0)-\min_{x\in\mathcal{X}}f(x)$, we have
\begin{align*}
    \min_{0\leq k\leq N-1} \Delta_k \leq \frac{4 L D_0}{N}.
\end{align*}
\end{prop}
\begin{proof}
Let $k\geq 0$, then by $L$-smoothness of $f$, we have
\begin{align*}
    f(x_{k+1})\leq f(x_k) +\langle\nabla f(x_k),x_{k+1}-x_k\rangle + L D_h(x_{k+1},x_k),
\end{align*}
and by optimality of $x_{k+1}$, we have for all $x\in\mathcal{X}$,
\begin{align*}
    \langle \nabla f(x_k)+\frac{1}{\gamma_k}[\nabla h(x_{k+1})-\nabla h(x_k)],x-x_{k+1}\rangle \geq 0,
\end{align*}
which implies, by taking $x=x_k$, that
\begin{align*}
     \langle \nabla f(x_k),x_k-x_{k+1}\rangle& \geq\frac{1}{\gamma_k}[-\langle\nabla h(x_{k+1}),x_k-x_{k+1}\rangle - \langle\nabla h(x_{k}),x_{k+1}-x_k\rangle]\\
     &\geq\frac{1}{\gamma_k}[D_h(x_k,x_{k+1})+D_h(x_{k+1},x_{k})].
\end{align*}
Then we have
\begin{align*}
    f(x_{k+1})\leq f(x_k)-\frac{1}{\gamma_k}[D_h(x_k,x_{k+1})+D_h(x_{k+1},x_{k})] + L D_h(x_{k+1},x_k)+L D_h(x_k,x_{k+1})
\end{align*}
where the last term is added by positivity of $D_h(\cdot,\cdot)$ (as $h$ is supposed to be convex on $\mathcal{X}$). Finally we obtain that
\begin{align*}
    \left(\sum_{k=0}^{N-1}\gamma_k(1-\gamma_k L)\Delta_k\right)\leq f(x_0)-f(x_N)\leq D_0,
\end{align*}
and as soon as $\gamma_k<\frac{1}{L}$, we have
\begin{align*}
    \min_{0\leq k\leq N-1} \Delta_k\leq \frac{D_0}{ \left(\sum_{k=0}^{N-1}\gamma_k(1-\gamma_k L)\right)}.
\end{align*}
Then by taking $\gamma_k=\frac{1}{2L}$, the result follows.
\end{proof}
In this paper, we consider $h$ to be the negative entropy function defined on $\Delta_q^{*}$ as
\begin{align}
    h(x)=\sum_{i=1}^q x_i\log(x_i).
\end{align}
Therefore the \emph{prox-function} associated is just the Kullback–Leibler divergence (KL) defined as,
\begin{align*}
      \text{KL}(x,z)=\sum_{i=1}^q x_i\log(x_i/z_i).
\end{align*}
Moreover if $\mathcal{X}\subset\prod_{i=1}^{p}\Delta_{q_i}^{*}$ for $p\geq 1$, we consider instead
\begin{align*}
    h((x^{(1)},\dots,x^{(p)})):=\sum_{i=1}^p\sum_{j=1}^{q_i}x_j^{(i)}\log(x_j^{(i)})
\end{align*}
where the associated \emph{prox-function} is 
\begin{align*}
    D_h((x^{(1)},\dots,x^{(p)}),(z^{(1)},\dots,z^{(p)}))=\sum_{i=1}^p \text{KL}(x^{(i)},z^{(i)}).
\end{align*}
\section{The Dykstra's Algorithm}
\label{sec-Dykstra}
In order to solve Eq.~(\ref{eq-barycenter-LOT-alpha}), we use the Dykstra's Algorithm~\cite{dykstra1983algorithm}. Given a closed convex set $\mathcal{C}\subset \mathbb{R}_{+}^{n\times r}\times \mathbb{R}_{+}^{m\times r} \times \mathbb{R}_{+}^{r}$, we denote for all $\bm{\xi}\in (\mathbb{R}_{+}^{*})^{n\times r}\times (\mathbb{R}_{+}^{*})^{m\times r} \times (\mathbb{R}_{+}^{*})^{r}$ the projection according to the Kullback-Leibler divergence as
\begin{align*}
    \mathcal{P}_{\mathcal{C}}^{\text{KL}}(\bm{\xi}) \eqdef \argmin_{\bm{\zeta}\in\mathcal{C}}\text{KL}(\bm{\zeta},\bm{\xi}).
\end{align*}
Starting from $\bm{\zeta}_0 \eqdef \bm{\xi}$ and $\bm{q}_{0}=\bm{q}_{-1}=(\mathbf{1},\mathbf{1},\mathbf{1})\in \mathbb{R}_{+}^{n\times r}\times \mathbb{R}_{+}^{m\times r} \times \mathbb{R}_{+}^{r}$, the Dykstra's Algorithm~\ref{alg-Dykstra} applied to our problem consists in computing for all $j\geq 0$, 
\begin{align*}
    \bm{\zeta}_{2j+1} &= \mathcal{P}_{\mathcal{C}_1(a,b,r,\alpha)}^{\text{KL}}(\bm{\zeta}_{2j}\odot \bm{q}_{2j-1})\\
    \bm{q}_{2j+1}& =  \bm{q}_{2j-1}\odot \frac{\bm{\zeta}_{2j}}{\bm{\zeta}_{2j+1}} \\
    \bm{\zeta}_{2j+2} &= \mathcal{P}_{\mathcal{C}_2(r)}^{\text{KL}}(\bm{\zeta}_{2j+1}\odot \bm{q}_{2j})\\
    \bm{q}_{2j+2}& =  \bm{q}_{2j} \odot \frac{\bm{\zeta}_{2j+1}}{\bm{\zeta}_{2j+2}}.
\end{align*}

In fact these operations can be simplified to simple matrix/vector multiplications. More precisely, the Dykstra's Algorithm produces the iterates $(\bm{\zeta}_j)_{j\geq 0}$ which satisfy for all $j\geq 0$ $\bm{\zeta}_j=(Q_j,R_j,g_j)$ where
\begin{align*}
    Q_j&= \Diag(u_j^{1})\xi^{(1)} \Diag(v_j^{1})\\
    R_j&= \Diag(u_j^{2})\xi^{(2)} \Diag(v_j^{2})
\end{align*}
for the sequences $(u_j^{i},v_j^{i})_{j\geq 0}$ initialized as, $u_0^{i} \eqdef \mathbf{1}_n$, $v_0^{i} \eqdef \mathbf{1}_m$ for all $i\in\{1,2\}$, $q_{0,1}^{(3)}=q_{0,2}^{(3)}= q_{0}^{(1)}=q_{0}^{(2)}=\mathbf{1}_r $ and computed with the iterations
\begin{align*}
    u_{n+1}^{k,i}&=\frac{p_i}{\xi_k^{i}v_n^{k,i}}\\
    \tilde{g}_{n+1} & = \max(\alpha,g_{n}\odot q_{n,1}^{(3)}),~
    q_{n+1,1}^{(3)} = (g_{n}\odot q_{n,1}^{(3)}) / \tilde{g}_{n+1}\\
    g_{n+1}&=(\tilde{g}_{n+1}\odot q_{n,2}^{(3)}  )^{1/3} \prod_{i=1}^2 (v_n^{k,i}\odot q_{n}^{(i)} \odot (\xi_{k}^{i})^Tu_n^{k,i})^{1/3}\\
     v_{n+1}^{k,i}&=\frac{g_{n+1}}{(\xi_k^{i})^T u_n^{k,i}}\\
     q_{n+1}^{(i)} & = (v_{n}^{k,i} \odot q_{n}^{(i)})/v_{n+1}^{k,i},~q_{n+1,2}^{(3)} = (\tilde{g}_{n+1}\odot q_{n,2}^{(3)})/g_{n+1}
\end{align*}


\section{Proofs}
\label{sec-proofs}
\subsection{Proof of Proposition~\ref{prop:existence-min}}
\begin{proof}
The case when $\varepsilon=0$ is clear. Assume now that $\varepsilon>0$.
When $r=1$, note that $\mathcal{C}_1(a,b,r)\cap\mathcal{C}_2(r)$ is closed as $g=1$ and bounded, therefore and by continuity of the objective the mininum exists. Let $r\geq 2$. First remarks that we always have $\text{LOT}_{r,\varepsilon}(\mu,\nu)\leq \text{LOT}_{r-1,\varepsilon}(\mu,\nu)$. Let us assume that (\ref{eq-LOT-reformulated-ent}) does not admits a minimum. Because the objective $F_{\varepsilon}$ is a lower semi-continuous function on $\overline{\mathcal{C}_1(a,b,r)}\cap\mathcal{C}_2(r)$, and by compacity of $\overline{\mathcal{C}_1(a,b,r)}\cap\mathcal{C}_2(r)$, the objective function admits a minimum $(Q,R,g)\in\overline{\mathcal{C}_1(a,b,r)}\cap\mathcal{C}_2(r)$ and we have $\text{LOT}_{r,\varepsilon}(\mu,\nu)=F_{\varepsilon}(Q,R,g)$. 
But as the minimum is not attained on $\mathcal{C}_1(a,b,r)\cap\mathcal{C}_2(r)$, it means that there exists at least one coordinate $i\in\{1,\dots,r\}$ such that $g_i=0$. Then because the constraints, $Q$ and $R$ both admit a column which is the null vector. By deleting these coordinates in $Q,R,g$, we obtain that  $\text{LOT}_{r,\varepsilon}(\mu,\nu)=\text{LOT}_{r-1,\varepsilon}(\mu,\nu)$. 
\end{proof}

\subsection{Proof of Proposition~\ref{prop:proj-C1}}

\begin{proof}
The first oder conditions of the projection gives that there exists $(\lambda_1,\lambda_2,\lambda_3)\in\mathbb{R}^{n}\times\mathbb{R}^{m}\times\mathbb{R}^r_{+}$ such that
\begin{align*}
    \log(Q/\tilde{Q}) +\lambda_1\mathbf{1}^T&=0\\
    \log(R/\tilde{R}) +\lambda_2\mathbf{1}^T&=0\\
    \log(g/\tilde{g}) +\lambda_3 &=0
\end{align*}
Moreover the conditions  $Q\mathbf{1}=a$, $R\mathbf{1}=b$ and $g\geq \alpha$ imply that
\begin{align*}
 Q&=\text{Diag}(a/\tilde{Q}\mathbf{1})\tilde{Q}\\
R &= \text{Diag}(b/\tilde{R}\mathbf{1})\tilde{R}\\
g & = \max(\alpha,\tilde{g}).
\end{align*}
\end{proof}

\subsection{Proof of Proposition~\ref{prop:proj-C2}}
\begin{proof}
The first order conditions of the projection states that there exists $(\lambda_1,\lambda_2)\in\mathbb{R}^r\times\mathbb{R}^r$ such that
\begin{align*}
    \log(Q/\tilde{Q}) + \mathbf{1}_n\lambda_1^T &= 0\\
    \log(R/\tilde{R}) + \mathbf{1}_m\lambda_2^T &= 0\\
    \log(g/\tilde{g}) - (\lambda_1 +\lambda_2) &= 0\\
\end{align*}
Moreover the conditions  $Q^T\mathbf{1}_n=R^T\mathbf{1}_m=g$ imply that
\begin{align*}
    Q&=\tilde{Q}\text{Diag}(g/\tilde{Q}^T\mathbf{1}_n)\\
R &= \tilde{R}\text{Diag}(g/\tilde{R}^T\mathbf{1}_m)\\
g^3 & =  \tilde{g}\odot \tilde{Q}^T\mathbf{1}_n \odot \tilde{R}^T\mathbf{1}_m
\end{align*}
from which the result follows.
\end{proof}

\subsection{Proof of Proposition~\ref{prop:cvg-MD-Dykstra}}

\begin{proof}
To show the result, we just need to show that 
$$F_{\varepsilon}:(Q,R,g)\in\mathcal{C}(a,b,r,\alpha)\rightarrow \langle C,Q\Diag (1/g) R^T\rangle -\varepsilon H(Q,R,g)$$ 
is smooth relatively to 
$$H(Q,R,g):=\sum_{i,j}Q_{i,j}\log(Q_{i,j})+\sum_{i,j}R_{i,j}\log(R_{i,j})+\sum_{j}g_{j}\log(g_{j}),$$
then by applying Proposition~\ref{prop-MD-general}, the result will follow. Let us now show that $F_{\varepsilon}$ is $L_{\varepsilon,\alpha}$-smooth. To do so, it is enough to show that~\cite{lu2017relativelysmooth,zhang2020wasserstein}
\begin{align*}
    \Vert \nabla F_\varepsilon(Q_1,R_1,g_1)-\nabla F_\varepsilon(Q_2,R_2,g_2)\Vert_2\leq L_{\varepsilon,\alpha}\Vert H(Q_1,R_1,g_1)- H(Q_2,R_2,g_2)\Vert_2 .
\end{align*}
We first have that
\begin{align*}
    \nabla F_\varepsilon(Q,R,g)=\left(CR\Diag(1/g)+\varepsilon(\log Q +\mathbf{1}),C^TQ\Diag(1/g)+\varepsilon(\log R +\mathbf{1}),
    -\mathcal{D}(Q^TRC)/g^2+\varepsilon(\log g+ 1)\right)
\end{align*}
Now we have,
\begin{align*}
 \Vert \nabla F_\varepsilon(Q_1)-\nabla F_\varepsilon(Q_2)\Vert_2^2&\leq\Vert CR_1\Diag (1/g_1)-CR_2\Diag (1/g_2)\Vert_2^2 +\varepsilon^2\Vert \log Q_1 - \log Q_2\Vert_2^2\\
 &+2\varepsilon\Vert \log Q_1-\log Q_2\Vert_2\Vert CR_1\Diag (1/g_1)-CR_2\Diag (1/g_2)\Vert_2\\
 &\leq \Vert C\Vert_2^2 \Vert (R_1 - R_2)\Diag (1/g_1)+(\Diag (1/g_1)-\Diag(1/g_2))R_2\Vert_2^2+\varepsilon^2\Vert \log Q_1 - \log Q_2\Vert_2^2\\
 &+2\varepsilon \Vert \log Q_1-\log Q_2\Vert_2\Vert CR_1\Diag (1/g_1)-CR_2\Diag (1/g_2)\Vert_2\\
 &\leq \Vert C\Vert_2^2\left[\frac{\Vert R_1-R_2\Vert_2^2}{\alpha^2}+\Vert 1/g_1-1/g_2\Vert_2^2+\frac{\Vert R_1-R_2\Vert \Vert 1/g_1-1/g_2\Vert_2}{\alpha} \right]+\varepsilon^2\Vert \log Q_1 - \log Q_2\Vert_2^2\\
 &+ 2\varepsilon \Vert \log Q_1-\log Q_2\Vert_2\Vert CR_1\Diag (1/g_1)-CR_2\Diag (1/g_2)\Vert_2.
\end{align*}
As $Q\rightarrow H(Q)$ is 1-strongly convex w.r.t to the $\ell_2$-norm on $\Delta_{n\times r}$, we have
\begin{align*}
    \Vert Q_1-Q_2\Vert_2^2&\leq \langle \log Q_1-\log Q_2,Q_1-Q_2\rangle \\
    &\leq \Vert  \log Q_1-\log Q_2\Vert_2\Vert Q_1-Q_2\Vert_2
\end{align*}
from which follows that
\begin{align*}
        \Vert Q_1-Q_2\Vert_2\leq \log Q_1-\log Q_2\Vert_2.
\end{align*}
Moreover we have
\begin{align*}
    \Vert 1/g_1 - 1/g_2\Vert_2\leq \frac{\Vert g_1-g_2\Vert_2}{\alpha^2}\leq \Vert \frac{\Vert \log g_1-\log g_2\Vert_2}{\alpha^2}
\end{align*}
Therefore we obtain that
\begin{align*}
\Vert \nabla F_\varepsilon(Q_1)-\nabla F_\varepsilon(Q_2)\Vert_2^2&\leq \left(\frac{\Vert C\Vert_2}{\alpha}\Vert \log R_1 - \log R_2\Vert_2 +\frac{\Vert C\Vert_2}{\alpha^2} \Vert\log g_1 - \log g_2\Vert_2 +\varepsilon\Vert \log Q_1 - \log Q_2\Vert_2\right)^2.
\end{align*}
An analogue proof leads to 
\begin{align*}
\Vert \nabla F_\varepsilon(R_1)-\nabla F_\varepsilon(R_2)\Vert_2^2&\leq \left(\frac{\Vert C\Vert_2}{\alpha}\Vert \log Q_1 - \log Q_2\Vert_2 +\frac{\Vert C\Vert_2}{\alpha^2} \Vert\log g_1 - \log g_2\Vert_2 +\varepsilon\Vert \log R_1 - \log R_2\Vert_2\right)^2.    
\end{align*}
Let us now consider smoothness of $F_\varepsilon$ w.r.t $g$,
\begin{align*}
\Vert \nabla F_\varepsilon(g_1)-\nabla F_\varepsilon(g_2)\Vert_2^2 &\leq \left\Vert \frac{\mathcal{D}(Q_1^TCR_1)}{g_1^2}-\frac{\mathcal{D}(Q_2^TCR_2)}{g_2^2}\right\Vert_2^2+\varepsilon^2\Vert \log g_1 - \log g_2\Vert_2^2\\ 
&+2\varepsilon \left\Vert \frac{\mathcal{D}(Q_1^TCR_1)}{g_1^2}-\frac{\mathcal{D}(Q_2^TCR_2)}{g_2^2}\right\Vert_2 \Vert \log g_1 - \log g_2\Vert_2.
\end{align*}
but we have that
\begin{align*}
 \left\Vert \frac{\mathcal{D}(Q_1^TCR_1)}{g_1^2}-\frac{\mathcal{D}(Q_2^TCR_2)}{g_2^2}\right\Vert_2^2 &\leq \Vert (1/g_1^2 - 1/g_2^2)\Diag (Q_1^TCR_1)\Vert_2^2 +\Vert \mathcal{D}(Q_1^TCR_1)-\mathcal{D}(Q_2^TCR_2)/g_2^2\Vert_2^2 \\
 &+2 \Vert (1/g_1^2 - 1/g_2^2)\Diag (Q_1^TCR_1)\Vert_2\Vert \mathcal{D}(Q_1^TCR_1)-\mathcal{D}(Q_2^TCR_2)/g_2^2\Vert_2\\
 &\leq \left(\frac{1\Vert C\Vert_2}{\alpha^2}\Vert \log g_1 - \log g_2\Vert_2+\frac{\Vert C\Vert_2}{\alpha^2}\left[ \Vert Q_1-Q_2\Vert_2^2+\Vert R_1-R_2\Vert_2\right]\right)^2.
\end{align*}
Therefore we obtain that 
\begin{align*}
   \Vert \nabla F_\varepsilon(g_1)-\nabla F_\varepsilon(g_2)\Vert_2^2\leq \left(\left(\frac{\varepsilon+2\Vert C\Vert_2}{\alpha^3}\right)\Vert \log g_1 - \log g_2\Vert_2 +\frac{\Vert C\Vert}{\alpha^2}\Vert Q_1 - Q_2\Vert_2 + +\frac{\Vert C\Vert}{\alpha^2}\Vert R_1 - R_2\Vert_2
   \right)^2
\end{align*}
Finally we obtain that
\begin{align*}
\Vert \nabla F_\varepsilon(Q_1,R_1,g_1)- \nabla F_\varepsilon(Q_2,R_2,g_2)\Vert_2^2
&\leq 3\left(\frac{\Vert C\Vert_2^2}{\alpha^2}+ \frac{\Vert C\Vert_2^2}{\alpha^4}+\varepsilon^2\right)[\Vert \log Q_1-\log Q_2\Vert_2^2+\Vert \log R_1-\log R_2\Vert_2^2] \\
&+ 3\left(\frac{2\Vert C\Vert_2^2}{\alpha^4}
+ \left(\frac{\varepsilon+2\Vert C\Vert_2}{\alpha^3}\right)^2 \right)\Vert \log g_1 -\log g_2\Vert_2^2
\end{align*}
Thus we obtain that
\begin{align*}
  \Vert \nabla F_\varepsilon(Q_1,R_1,g_1)- \nabla F_\varepsilon(Q_2,R_2,g_2)\Vert_2\leq L_{\varepsilon,\alpha}  \Vert \nabla H(Q_1,R_1,g_1)- \nabla H(Q_2,R_2,g_2)\Vert_2
\end{align*}
and the result follows.
\end{proof}

\section{Low-Rank Factorization of Distance Matrix}
\label{sec-LR-Distance}
In this section we present the algorithm used to perform a low-rank approximation of a distance matrix~\cite{bakshi2018sublinear,indyk2019sampleoptimal}. Given a metric space $(\mathcal{X},d)$, $X=\{x_i\}_{i=1}^n \in\mathcal{X}^n$ and $Y=\{y_j\}_{j=1}^m\in \mathcal{X}^m$ we aim at obtaining a low-rank approximation of the distance matrix $D=(d(x_i,y_j))_{i,j}$ with a precision $\gamma>0$. Let us now present the algorithm considered where we have denoted $t=\lfloor r/\gamma\rfloor$.
\begin{algorithm}[H]
\SetAlgoLined
\textbf{Inputs:} $X,Y,r,\gamma$\\
Choose $i^{*}\in\{1,\dots,n\}$, and $j^*\{1,\dots,m\}$ uniformly at random.\\
For $i=1,\dots,n$,~$p_i\gets d(x_i,y_j^*)^2 +  d(x_i^*,y_j^*)^2+\frac{1}{m}\sum_{j=1}^m d(x_i^*,y_j)^2$.\\
Independently choose $i^{(1)},\dots,i^{(t)}$ according $(p_1,\dots,p_n)$.\\
$X^{(t)}\gets [x_{i^{(1)}},\dots,x_{i^{(t)}}],~P^{(t)}\gets [\sqrt{t p_{i^{(1)}}},\dots,\sqrt{t p_{i^{(t)}}}],~S\gets d(X^{(t)},Y)/P^{(t)}$\\
Denote $S=[S^{(1)},\dots,S^{(m)}]$,\\
For $j=1,\dots,m$,~$q_j\gets \Vert S^{(j)} \Vert_2^2/\Vert S\Vert_F^2$\\
Independently choose $j^{(1)},\dots,j^{(t)}$ according $(q_1,\dots,q_m)$.\\
$S^{(t)}\gets  [S^{j^{(1)}},\dots,S^{j^{(t)}}],~Q^{(t)}\gets [\sqrt{t q_{j^{(1)}}},\dots,\sqrt{t q_{j^{(t)}}}],~W\gets S^{(t)}/Q^{(t)}$\\
$U_1,D_1,V_1 \gets \text{SVD}(W)$ (decreasing order of singular values).\\
$N\gets [U_1{(1)},\dots,U_1^{(r)}],~N\gets S^T N/\Vert W^T N\Vert_F$\\
Choose $j^{(1)},\dots,j^{(t)}$ uniformly at random in $\{1,\dots,m\}$.\\
$Y^{(t)}\gets [y_{j^{(1)}},\dots,y_{j^{(t)}}], D^{(t)}\gets d(X,Y^{(t)})/\sqrt{t}$.\\
$U_2,D_2,V_2 = \text{SVD}(N^T N),~U_2\gets U_2/D_2,~N^{(t)} \gets [(N^T)^{(j^{(1)})},\dots,(N^T)^{(j^{(t)})}] ,~B\gets U_2^T N^{(t)}/\sqrt{t},~A\gets (BB^T)^{-1}$.\\
$Z\gets AB(D^{(t)})^T,~M\gets Z^{T}U_2^T$\\
\textbf{Result:} $M,N$
\caption{$\text{LR-Distance}(X,Y,r,\gamma)$ \label{alg-LR-distance}}
\end{algorithm}

\section{Positive low-rank factorization with fixed marginal}
\label{sec-LR-fixed}
Let $g\in\Delta_r^{*}$, and let us for now consider the following problem
\begin{align}
\label{eq-LOT-fixed}
    \text{LOT}_{r,g}(\mu,\nu):=\min_{P\in\Pi_{a,g,b}}\langle C,P\rangle.
\end{align}
By definition of $\Pi_{a,g,b}$, this problem can be formulated as follows:
\begin{align}
\label{eq-LOT-fixed-refomulated}
    \text{LOT}_{r,g}(\mu,\nu)=\min_{\substack{Q\in\Pi_{a,g}\\ R\in\Pi_{b,g}}}\langle C,Q\text{Diag}(1/g)R^T\rangle.
\end{align}
As in the classical OT problem, one can extend the above objective and consider for any $\varepsilon\geq 0$ an entropic version of the problem defined as
\begin{equation}
\begin{aligned}
\label{eq-LOT-fixed-refomulated-reg}
    \text{LOT}_{r,g,\varepsilon}(\mu,\nu):=\min_{\substack{Q\in\Pi_{a,g}\\ R\in\Pi_{b,g}}}\langle C,Q\text{Diag}(1/g)R^T\rangle
     - \varepsilon H((Q,R))
\end{aligned}
\end{equation}
Note that for any $\varepsilon\geq 0$, the minimum always exists as the objective is continuous and $\Pi_{a,g,b}$ is compact. Moreover we clearly have that $\text{LOT}_{r,g,0}(\mu,\nu)= \text{LOT}_{r,g}(\mu,\nu)$. Applying a MD method to the objective~(\ref{eq-LOT-fixed-refomulated}) leads for all $k\geq 0$ to the following updates
\begin{align*}
    Q_{k+1}&:=\argmin_{Q\in\Pi_{a,g}}\langle C_k^{(1)},Q\rangle -\frac{1}{\gamma_k}H(Q)\\
        R_{k+1}&:=\argmin_{R\in\Pi_{a,g}}\langle C_k^{(2)} R\rangle -\frac{1}{\gamma_k}H(R)
\end{align*}
where, $(Q_0,R_0)\in\Pi_{a,g}\times\Pi_{b,g}$ is an initial point, $C_k^{(1)}:=CR_k\text{Diag}(1/g) + (\varepsilon-\frac{1}{\gamma_k})\log(Q_k)$,  $C_k^{(2)}:=C^TQ_k\text{Diag}(1/g)+ (\varepsilon-\frac{1}{\gamma_k})\log(R_k)$ and $\gamma_k$ is a sequence of positive real numbers. Therefore a MD method bowls down to solve at each iteration two regularized OT problems which can be done efficiently using the Sinkhorn algorithm~(\ref{alg-sink}).

\paragraph{Convergence of the Mirror Descent.} Even if the objective~(\ref{eq-LOT-fixed-refomulated}) is not convex in $(Q,R)$, one can obtain the non-asymptotic stationary convergence of the MD algorithm in this setting. 


Let $f_{\varepsilon}$ be the objective function of the problem~(\ref{eq-LOT-fixed-refomulated-reg}) defined on $X:=\Pi_{a,g}\times\Pi_{b,g}$ and let us denotes for any $\gamma>0$ and $x\in X$ 
\begin{align*}
    \mathcal{G}_\varepsilon(x,\gamma):=\argmin_{u\in X}\{ \langle \nabla f_{\varepsilon}(x),u \rangle +\frac{1}{\gamma} KL(u,x) \}.
\end{align*}
Let us now define the following criterion to establish convergence:
\begin{align*}
   \Delta_\varepsilon(x,\gamma):=\frac{1}{\gamma^2}(KL(x,\mathcal{G}_\varepsilon(x,\gamma))+KL(\mathcal{G}_\varepsilon(x,\gamma),x)).
\end{align*}

To show the non-asymptotic stationary convergence, we show that for any $\varepsilon\geq 0$, the objective is smooth relative to the entropy function~\cite{bauschke2017descent} and we extend the proof of~\cite{ghadimi2013minibatch} to this case. 
\begin{prop*}
Let $\varepsilon\geq 0$ and $N\geq 1$. By denoting 
$L_\varepsilon:=\sqrt{2(\Vert C\Vert_2^2\Vert\text{Diag}(1/g)\Vert_2^2+\varepsilon^2)}$ and by considering a constant stepsize in the MD scheme such that for all $k=1,\dots,N$ $\gamma_k=\frac{1}{L_\varepsilon}$, we obtain that
\begin{align*}
    \min_{1\leq k\leq N}\Delta_\varepsilon((Q_k,R_k),\gamma_k)\leq \frac{2L_\varepsilon  D_0}{N}.
\end{align*}
where $D_0:= f_{\varepsilon}(Q_0,R_0) - \text{LOT}_{r,g,\varepsilon}$ is the distance of the initial value to the optimal one. 
\end{prop*}
\begin{proof}
A similar proof of the one given for Proposition~\ref{prop:cvg-MD-Dykstra} gives that $f_\varepsilon$ is $L_\varepsilon$-smooth relatively to $H$.
\end{proof}

Let us now introduce our first algorithm~(\ref{alg-LOT-F}) to compute a positive low-rank factorization of the optimal coupling. Here we consider the case where $g:=\mathbf{1}_r/r$. Before introducing our algorithm it is worth noting that a trivial initialization may lead to a trivial fixed point in the MD updates. Indeed if one initialize $Q:=ag^T$ and $R:=bg^T$, then  $CR\text{Diag}(1/g)=Ca\mathbf{1}^T$ and $C^TQ\text{Diag}(1/g)=C^Tb\mathbf{1}^T$ and therefore $(Q,R)$ is a fixed point of the MD. To avoid this, we initialize our algorithm in the following way: let $\lambda:=\min_{i,j,k}(a_i,b_j,g_k)/2$, 
$a_1\in\Delta_{n}^{*}\backslash\{a\}$, $a_2:= (a - \lambda a_1)/(1-\lambda)$, $b_1\in\Delta_{n}^{*}\backslash\{b\}$, $b_2:= (b - \lambda b_1)/(1-\lambda)$, $g_1\in\Delta_{r}^{*}\backslash\{g\}$ and $g_2:= (g - \lambda g_1)/(1-\lambda)$. We can now define our initialization as $Q:=\lambda a_1 g_1^T+ (1-\lambda) a_2g_2^T$, $R:=\lambda b_1 g_1^T+ (1-\lambda) b_2g_2^T$.

\begin{algorithm}[H]
\SetAlgoLined
\textbf{Inputs:} $C,a,b,\delta,Q,R,g,\gamma,\delta_{\text{S}}$\\
\Repeat{$ \Delta((Q,R),\gamma) <\delta$}{
    $Q_{\text{old}}\gets Q,~R_{\text{old}}\gets R\\
    C^{(1)}\gets CR\text{Diag}(1/g)-\frac{1}{\gamma}\log(Q),\\
    C^{(2)}\gets C^TQ\text{Diag}(1/g)-\frac{1}{\gamma}\log(R),\\ 
    K^{(1)}\gets \exp(-\gamma C^{(1)}),\\
    K^{(2)}\gets \exp(-\gamma C^{(2)}),\\
    u,v\gets \text{Sinkhorn}(K^{(1)},a,g,\delta_{\text{S}})~(\text{Algorithm}~(\ref{alg-sink})),\\
    Q\gets \text{Diag}(u) K^{(1)}  \text{Diag}(v), \\
    u,v\gets \text{Sinkhorn}(K^{(2)},a,g,\delta_{\text{S}})~(\text{Algorithm}~(\ref{alg-sink})),\\
    R\gets \text{Diag}(u) K^{(2)}\text{Diag}(v)$
  }
\textbf{Result:} $Q,R$
\caption{$\text{LOT-F}(C,a,b,\delta)$ \label{alg-LOT-F}}
\end{algorithm}

\paragraph{Computational Cost.} Note that the kernels $(K^{(i)})_{1\leq i\leq 2}$ considered in algorithm~(\ref{alg-LOT-F}) live in $\mathbb{R}_{+}^{n\times r}\times\mathbb{R}_{+}^{m\times r}$ and therefore each iteration of both Sinkhorn algorithms can be computed either in $\mathcal{O}(nr)$ or in $\mathcal{O}(mr)$ algebraic operations as it involves only matrix/vector multiplications of the form $K^{(i)}v$ and $(K^{(i)})^Tu$. However without any assumption on the cost matrix $C$, computing $(K^{(i)})_{1\leq i\leq 2}$ costs $\mathcal{O}(nmr)$ algebraic operations as it requires to compute both $CR$ and $C^TQ$ at each iteration. Thanks to assumption~\ref{assump-low-rank-sin}, such multiplications can be performed in $\mathcal{O}((n+m)dr)$ algebraic operations and thus algorithm~(\ref{alg-LOT-F}) requires only a linear number of algebraic operations with respect to the number of samples at each iteration.

In the following, we will see that if we do not fix the marginal, the problem can also be solved efficiently as each iteration of the MD algorithm can be seen as a wasserstein barycenter problem.

\section{A Positive low-rank factorization with free marginal}
\label{sec-LR-IBP}
\label{sec-IBP}
Applying a MD method to the objective~(\ref{eq-LOT-reformulated-ent}) leads, for all $k\geq 0$, to the following updates 
\begin{equation}
\label{eq-barycenter-LOT}  
 (Q_{k+1},R_{k+1},g_{k+1}):= \argmin_{\bm{\zeta} \in\mathcal{C}_1(a,b,r)\cap \mathcal{C}_2(r)} \text{KL}(\bm{\zeta},\bm{\xi}_k)
\end{equation}
where $(Q_0,R_0,g_0)\in\mathcal{C}_1(a,b,r)\cap\mathcal{C}_2(r)$ is an initial point,
$\bm{\xi}_k:=(\xi_{k}^{(1)},\xi_{k}^{(2)},\xi_{k}^{(3)})$, $\xi_{k}^{(1)}:=\exp(-\gamma_kCR_k\text{Diag}(1/g_k)_k- (\gamma_k\varepsilon-1)\log(Q_k))$, $\xi_{k}^{(2)}:=\exp(-\gamma_kC^TQ_k\text{Diag}(1/g_k)- (\gamma_k\varepsilon-1)\log(R_k))$,
$\xi_{k}^{(3)}:=\exp(\gamma_k\omega_k/g_k^2- (\gamma_k\varepsilon-1)\log(g_k))$ with  $[\omega_k]_i:=[Q_k^TCR_k]_{i,i}$ for all $i\in\{1,\dots,r\}$ and $(\gamma_k)_{k\geq 0}$ is a sequence of positive real numbers. 

Eq.~(\ref{eq-barycenter-LOT}) is well defined. Indeed as the kernels $(\xi_{k}^{(i)})$ are matrices with positive coefficients, the infimum is attained in $\mathcal{C}_1(a,b,r)\cap\mathcal{C}_2(r)$ and the problem admits a unique solution. Moreover solving Eq.~(\ref{eq-barycenter-LOT}) bowls down to solve
\begin{equation}
\label{reformulation-update-MD-LOT}
 (Q_{k+1},R_{k+1},g_{k+1}):= \\
 \argmin_{\bm{\zeta}\in\overline{\mathcal{C}_1(a,b,r)}\cap \mathcal{C}_2(r)} \text{KL}(\bm{\zeta},\bm{\xi}_k)
\end{equation}
In order to solve Eq.~(\ref{reformulation-update-MD-LOT}), we consider the Iterative Bregman Projections (IBP) algorithm. Starting from $\bm{\zeta}^{(k)}_0:=\bm{\xi}_k$, the IBP algorithm consists in computing for all $j\geq 0$, 
\begin{align*}
    \bm{\zeta}_{2j+1}^{(k)} &= \mathcal{P}_{\overline{\mathcal{C}_1(a,b,r)}}^{\text{KL}}(\bm{\zeta}_{2j}^{(k)})\\
    \bm{\zeta}_{2j+2}^{(k)} &= \mathcal{P}_{\mathcal{C}_2(r)}^{\text{KL}}(\bm{\zeta}_{2j+1}^{(k)}).
\end{align*}
As $\overline{\mathcal{C}_1(a,b,r)}$ and $\mathcal{C}_2(r)$ are affine subspaces (note that nonnegativity constraints are already in the definition of the objective) one can show that $\bm{\zeta}_{j}^{(k)}$ converges towards the unique solution of Eq.~(\ref{reformulation-update-MD-LOT}),~\cite{bregman1967relaxation}. Remarks that the projection on $\overline{\mathcal{C}_1(a,b,r)}$ can be computed very easily as one has for any $\tilde{\bm{\xi}}:=(\tilde{Q},\tilde{R},\tilde{g})\in\mathbb{R}_{+}^{n\times r}\times \mathbb{R}_{+}^{n\times r} \times \mathbb{R}_{+}^{r}$, 
\begin{align*}
    \mathcal{P}_{\overline{\mathcal{C}_1(a,b,r)}}^{\text{KL}}(\tilde{\bm{\xi}})=\left(\text{Diag}\left(\frac{a}{\tilde{Q}\mathbf{1}_r}\right)\tilde{Q},\text{Diag}\left(\frac{b}{\tilde{R}\mathbf{1}_r}\right)\tilde{R},\tilde{g}\right)
\end{align*}
and the solution of the projection on $\mathcal{C}_2(r)$ is already given in Proposition~\ref{prop:proj-C2}.
\paragraph{Efficient computation of the updates.}  For all $k\geq 0$, starting with $\bm\zeta_0^{(k)}:=\bm{\xi}_k$ the IBP algorithm leads to a simple algorithm~(\ref{alg-IBP}) which computes only scaling vectors. More precisely, the IBP algorithm produces the iterates $(\bm{\zeta}_n^{(k)})_{n\geq 0}$ which satisfy for all $n\geq 0$ $\bm{\zeta}_n^{(k)}=(Q_n^{(k)},R_n^{(k)},g_n^{(k)})$ where
\begin{align*}
    Q_n^{(k)}&=\text{Diag}(u_n^{k,1})\xi_{k}^{1}\text{Diag}(v_n^{k,1})\\
    R_n^{(k)}&=\text{Diag}(u_n^{k,2})\xi_{k}^{2}\text{Diag}(v_n^{k,2})
\end{align*}
for the sequences $(u_n^{k,i},v_n^{k,i})$ initialized as $v_0^{k,i}:=\mathbf{1}$ for all $i\in\{1,2\}$ and computed with the iterations
\begin{align*}
    u_n^{k,i}&=\frac{p_i}{\xi_k^{i}v_n^{k,i}}\\
    g_{n+1}^{(k)}&=(g_{n}^{(k)})^{1/3} \prod_{i=1}^2 (v_n^{k,i}\odot(\xi_{k}^{i})^Tu_n^{k,i})^{1/3}\\
     v_{n+1}^{k,i}&=\frac{g_{n+1}^{(k)}}{(\xi_k^{i})^T u_n^{k,i}}
\end{align*}
where we have denoted $p_1:=a$ and $p_2:=b$ to simplify the notations. 

\begin{algorithm}[H]
\SetAlgoLined
\textbf{Inputs:} $\xi^{(1)},\xi^{(2)},g:=\xi^{(3)},p_1,p_2,\delta, v^{(i)}$\\
\Repeat{$\sum_{i=1}^2\|u^{(i)}\odot \xi^{(i)}v^{(i)} - p_i\|_1 <\delta$}{
    $u^{(i)}\gets p_i/\xi^{(i)}v^{(i)}~\forall i\in\{1,2\},\\
    g\gets (g)^{1/3} \prod_{i=1}^2 (v^{(i)}\odot(\xi^{(i)})^Tu^{(i)})^{1/3},\\ v^{(i)}\gets g/(\xi^{(i)})^T u^{(i)} ~\forall i\in\{1,2\}$
  }
$Q\gets \text{Diag}(u^{(1)})\xi_{k}^{(1)}\text{Diag}(v^{(1)})$\\
$R\gets \text{Diag}(u^{(2)})\xi_{k}^{(2)}\text{Diag}(v^{(2)})$\\
\textbf{Result:} $Q,R,g$
\caption{$\text{LR-IBP}((\xi^{(i)})_{1\leq i\leq 3},p_1,p_2,\delta)$ \label{alg-IBP}}
\end{algorithm}

Let us now introduce the proposed MD algorithm applied to~(\ref{eq-LOT-reformulated}). By denoting $\mathcal{D}(\cdot)$ the operator extracting the diagonal of a square matrix we  obtain the following algorithm~(\ref{alg-MDLROT}) to solve Eq.~(\ref{eq-LOT}). We initialize our algorithm with the exact same procedure as in algorithm~(\ref{alg-LOT-F}).
\begin{algorithm}[H]
\SetAlgoLined
\textbf{Inputs:} $C,a,b,(\gamma_k)_{k\geq 0}, Q,R,g,\delta$\\
\For{$k=1,\dots$}{
    $\xi^{(1)}\gets\exp(-\gamma_kCR\text{Diag}(1/g)- (\gamma_k\varepsilon-1)\log(Q)),\\ \xi^{(2)}\gets\exp(-\gamma_kC^TQ\text{Diag}(1/g)- (\gamma_k\varepsilon-1)\log(R)),\\
    \omega\gets \mathcal{D}(Q^TCR),~\xi^{(3)}\gets\exp(\gamma_k\omega/g^2- (\gamma_k\varepsilon-1)\log(g)),\\
    Q,R,g\gets \text{LR-IBP}((\xi^{(i)})_{1\leq i\leq 3},a,b,\delta)~(\text{Algorithm~(\ref{alg-IBP})})$
  }
\textbf{Result:} $\langle C,Q\text{Diag}(1/g)R^T\rangle$
\caption{$\text{LOT}(C,a,b,r,\delta)$ \label{alg-MDLROT}}
\end{algorithm}

\paragraph{Computational Cost.} Note that $(\xi^{(i)})_{1\leq i\leq 3}$ considered in algorithm~(\ref{alg-MDLROT}) lives in $\mathbb{R}_{+}^{n\times r}\times\mathbb{R}_{+}^{m\times r}\times \times\mathbb{R}_{+}^{r} $ and therefore each iteration of algorithm~(\ref{alg-IBP}) can be computed in $\mathcal{O}((n+m)r)$  algebraic operations as it involves only matrix/vector multiplications of the form $\xi^{(i)}v_i$ and $(\xi^{(i)})^Tu_i$. However without any assumption on the cost matrix $C$, computing $(\xi^{(i)})_{1\leq i\leq 3}$ costs $\mathcal{O}(nmr)$ algebraic operations as it requires to compute both $CR$ and $C^TQ$ at each iteration. Thanks to assumption~\ref{assump-low-rank-sin}, such multiplications can be performed in $\mathcal{O}((n+m)dr)$ algebraic operations and thus algorithm~(\ref{alg-MDLROT}) requires only a linear number of algebraic operations with respect to the number of samples at each iterations.

\section{Addiational Experiments}
\label{sec-exp-add}
In Fig.~\ref{fig-LR-Euclidean}, we compare two Gaussian mixture densities sampled with $n=m=10000$ points in 2D. The two densities considered are
\begin{align*}
    f_X(x)&=\frac{1}{3}\frac{\exp\left((x-\mu_1)^T\Sigma^{-1}(x-\mu_1)\right)}{\sqrt{2\pi|\Sigma|}} + \frac{1}{3}\frac{\exp\left((x-\mu_2)^T\Sigma^{-1}(x-\mu_2)\right)}{\sqrt{2\pi|\Sigma|}} + \frac{1}{3}\frac{\exp\left((x-\mu_3)^T\Sigma^{-1}(x-\mu_3)\right)}{\sqrt{2\pi|\Sigma|}}\\
        f_Y(x)&=\frac{1}{2}\frac{\exp\left((x-\nu_1)^T\Sigma^{-1}(x-\nu_1)\right)}{\sqrt{2\pi|\Sigma|}} + \frac{1}{2}\frac{\exp\left((x-\nu_2)^T\Sigma^{-1}(x-\nu_2)\right)}{\sqrt{2\pi|\Sigma|}}
\end{align*}
where 
\begin{align*}
    &\mu_1 = [0,0],\quad~\mu_2=[0,1],\quad~\mu_3=[1,1],\quad~\nu_1 = [0.5,0.5],\quad~\nu_2=[-0.5,0.5],~\quad \Sigma =0.05\times\text{Id}_2.
\end{align*}
We show in Fig.~\ref{fig-data-mix} a plot of the two distributions considered.
\begin{figure*}[!h]
\centering
\includegraphics[width=0.5\textwidth]{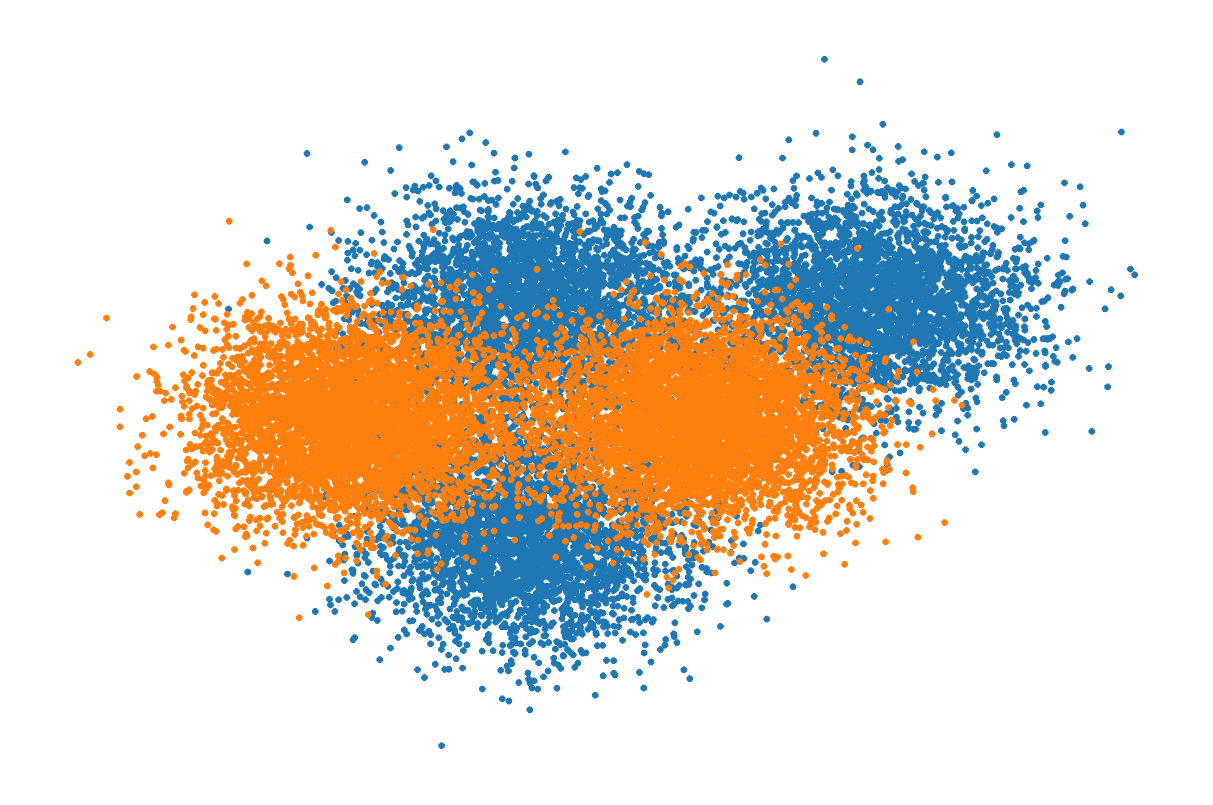}
\caption{Plot of the Gaussian mixtures considered in Fig.~\ref{fig-LR-Euclidean}.}\label{fig-data-mix}
\vspace{-0.1cm}
\end{figure*}
In Fig.~\ref{fig-LR-euclidean-d-10}, we consider the exact same setting as the one presented in Fig.~\ref{fig-LR-Euclidean} but we increase the dimension of the problem. More precisely we consider two Gaussian mixture densities samples with $n=m=10000$ points in 10D where
\begin{align*}
    &\mu_1 = [0,\dots,0],~\mu_2=[0,1,0,\dots,0],~\mu_3=[1,1,0,\dots,0],\\
    &~\nu_1 = [0.5,0.5,0,\dots,0],~\nu_2=[-0.5,0.5,0,\dots,0],\\
    &\Sigma =0.05\times\text{Id}_{10}.
\end{align*}
Similarly as in Fig.~\ref{fig-LR-Euclidean}, we observe that   \textbf{LOT} and \textbf{LOT Quad} provide similar results while \textbf{LOT} is faster. All kernel-based methods fail to converge in this setting. Moreover we see that for small regularizations $\varepsilon$, our method is able to approximate faster than \textbf{Sin} the true OT thanks to the low-rank constraint. Note also that we observe again a difference between the two entropic regularizations of the \textbf{Sin} objective and \textbf{LOT} objective. Indeed the range of $\varepsilon$ where \textbf{Sin} provides an efficient approximation of the true OT is larger than the one of \textbf{LOT}. Indeed recall that for $\textbf{LOT}$, we regularize \emph{twice} as we constraint the nonnegative rank of the couplings and we add an entropic term to regularize the objective.

\begin{figure*}[!h]
\centering
\includegraphics[width=1\textwidth]{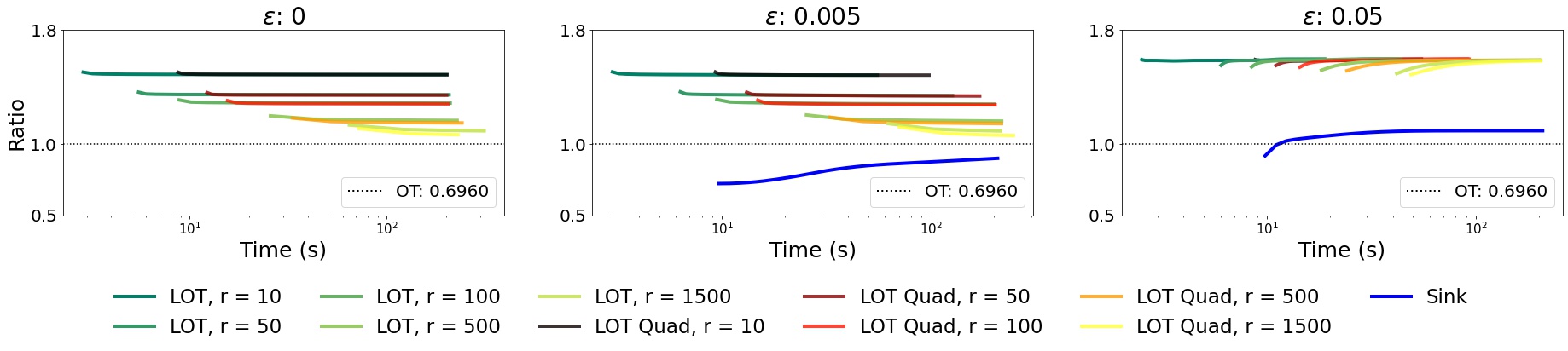}
\caption{Comparison of the time-accuracy tradeoff for different methods for estimating the OT or its regularized version between two mixture of gaussians in 10D.}\label{fig-LR-euclidean-d-10}
\vspace{-0.1cm}
\end{figure*}

In Fig.~\ref{fig-LR-Square-Euclidean}, we compare the time-accuracy tradeoff for different methods on a synthetic problem where we aim at estimating either the OT or its regularized version between two gaussians in 2D. Here we consider the exact same setting but we increase the dimension of the problem: $d=10$. As in Fig.~\ref{fig-LR-Square-Euclidean}, our proposed method obtains an efficient approximation of the OT or its regularized version for all rank $r$ faster than other low-rank methods in the regime of small $\varepsilon$. We also see that for all low-rank methods, a rank of $r=500$ is not enough in this setting to obtain the exact OT, but as the rank increases, the approximation gets better.
\begin{figure*}[!h]
\centering
\includegraphics[width=1\textwidth]{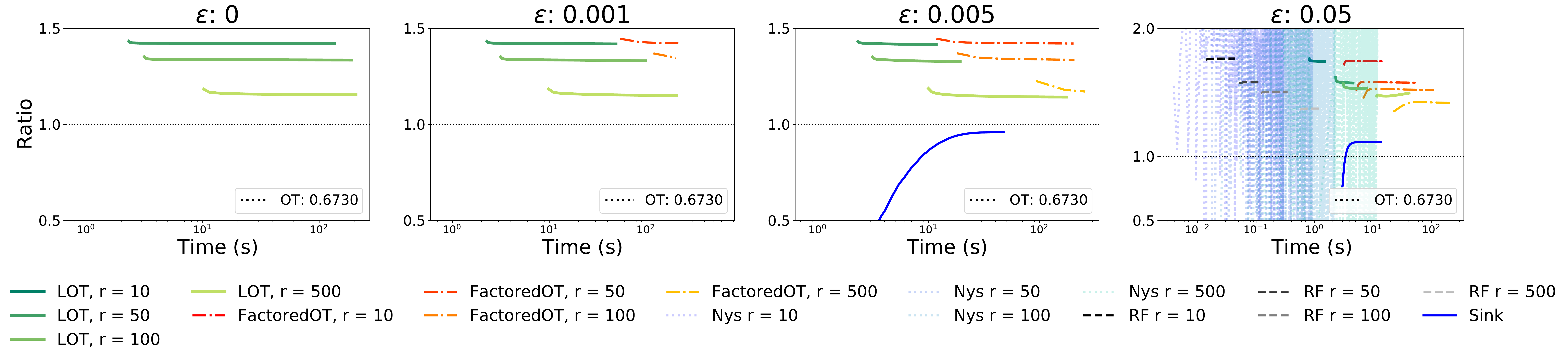}
\caption{In this experiment, we consider two Gaussian distributions evaluated on $n=m=5000$ in 10D. The first one has a mean of $(1,\dots,1)^T\in\mathbb{R}^{10}$ and identity covariance matrix $I_{10}$ while the other has 0 mean and covariance $0.1\times I_{10}$. The ground cost is the squared Euclidean distance.
}\label{fig-LR-Square-Euclidean-d-10}
\vspace{-0.1cm}
\end{figure*}

\end{document}